%% file: main.tex
\documentclass[11pt]{article}
\usepackage[margin=1in]{geometry}

\usepackage{macros}

\usepackage{etoolbox}

\bluehyperref

\newtoggle{arxiv}
\toggletrue{arxiv}

\usepackage[style=alphabetic,backend=bibtex,maxnames=12,maxbibnames=10,maxcitenames=10,maxalphanames=10,giveninits=true,doi=false,url=true]{biblatex}
\newcommand*{\citet}[1]{\AtNextCite{\AtEachCitekey{\defcounter{maxnames}{2}}} \textcite{#1}}

\newcommand*{\citep}[1]{\cite{#1}}
\newcommand{\citeyearpar}[1]{\cite{#1}}
\bibliography{bib}
\let\citealp\citep

\usepackage{overpic}

\usepackage{macros-paper}

\usepackage[noend]{algorithmic}
\usepackage{algorithm}

\usepackage{comment}

\usepackage{cellspace, tabularx}                 

\newcolumntype{C}{>{\centering\arraybackslash}X}
\addparagraphcolumntypes{C}
\newcolumntype{P}[1]{>{\arraybackslash}p{#1}}
\usepackage{array}
\usepackage{multirow}
\renewcommand{\arraystretch}{3}
\newcolumntype{x}[1]{%
	>{\raggedleft\hspace{0pt}}p{#1}}%
\usepackage{soul}

\newif\ifcomments
\commentsfalse
\ifcomments
\newcommand{\ha}[1]{
		\textcolor{blue}{\textbf{HA:} {#1}}
}

\newcommand{\jn}[1]{
		\textcolor{olive}{\textbf{JN:} {#1}}
}

\newcommand{\vf}[1]{
		\textcolor{magenta}{\textbf{VF:} {#1}}
}

\newcommand{\hn}[1]{
		\textcolor{red}{\textbf{HN:} {#1}}
}

\newcommand{\knote}[1]{
\textcolor{red}{\textbf{KT:} {#1}}
}

\else
\newcommand{\ha}[1]{}

\newcommand{\jn}[1]{}

\newcommand{\vf}[1]{}

\newcommand{\hn}[1]{}

\newcommand{\knote}[1]{}
\fi
\title{Fast Optimal Locally Private Mean Estimation \\ via Random Projections}
\author{Hilal Asi\thanks{Apple Inc. \texttt{hilal.asi94@gmail.com}.} \and Vitaly Feldman\thanks{Apple Inc. \texttt{vitaly.edu@gmail.com}.} \and Jelani Nelson\thanks{UC Berkeley. \texttt{minilek@berkeley.edu}. Supported by NSF grant CCF-1951384, ONR grant N00014-18-1-2562, and ONR DORECG award N00014-17-1-2127.} \and Huy L. Nguyen\thanks{Northeastern. \texttt{huylenguyen@gmail.com}.  Supported in part by NSF CAREER grant CCF-1750716.} \and Kunal Talwar\thanks{Apple Inc. \texttt{kunal@kunaltalwar.org}.}}

\begin{document}

\maketitle

\input{abstract.tex}

\input{introduction.tex}

\input{algorithms.tex}

\input{algs-fixed-transformation}

\input{experiments.tex}

\printbibliography

\newpage
\appendix
\onecolumn

\input{privUnitG}

\input{algorithms-unbiased.tex}

\input{appendix-gaussian.tex}

\section{Helper Lemmas}
\input{appendix-rotation.tex}

\input{appendix-srht.tex}

\input{appendix-compressed-privunit.tex}


\end{document}

%% file: abstract.tex
\begin{abstract}
We study the problem of locally private mean estimation of high-dimensional vectors in the Euclidean ball. Existing algorithms for this problem either incur sub-optimal error or have high communication and/or run-time complexity. We propose a new algorithmic framework, \pjl, for private mean estimation that yields algorithms that are computationally efficient, have low communication complexity, and incur optimal error up to a $1+o(1)$-factor. Our framework is deceptively simple: each randomizer projects its input to a random low-dimensional subspace, normalizes the result, and then runs an optimal algorithm such as \pug in the lower-dimensional space. In addition, we show that, by appropriately correlating the random projection matrices across devices, we can achieve fast server run-time. We mathematically analyze the error of the algorithm in terms of properties of the random projections, and study two instantiations. Lastly, our experiments for private mean estimation and private federated learning demonstrate that our algorithms empirically obtain nearly the same utility as optimal ones while having significantly lower communication and computational cost.
\end{abstract}

%% file: introduction.tex
\section{Introduction}
Distributed estimation of the mean, or equivalently the sum, of vectors $v_1,\ldots,v_n\in\r^d$ is a fundamental problem in distributed optimization and federated learning. For example, in the latter, each of $n$ devices may compute some update to parameters of a machine learning model based on its local data, at which point a central server wishes to apply all updates to the model, i.e.\ add $\sum_{i=1}^n v_i$ to the vector of parameters. The typical desire to keep local data private necessitates methods for computing this sum while preserving privacy of the local data on each individual device, so that the central server essentially only learns the noisy sum and (almost) nothing about each individual summand $v_i$~\cite{DworkKeMcMiNa06,BonawitzIvKrMaMcPaRaSeSe17}.

The gold standard for measuring privacy preservation is via the language of {\it differential privacy} \cite{DworkMcNiSm06}. In this work, we study this problem in the setting of {\it local differential privacy} (LDP). We consider (one-round) protocols for which there exists some randomized algorithm $\R:\r^d\rightarrow \mathcal M$ (called the {\it local randomizer}), such that each device $i$ sends $\R(v_i)$ to the server. We say the protocol is {\it $\diffp$-differentially private} if for any $v, v'\in\r^d$ and any event $S\subseteq \mathcal M$,
$\Pr(\R(v) \in S) \le e^\diffp \Pr(\R(v') \in S)$.
If $\diffp = 0$ then the distribution of $\R(v)$ is independent of $v$, and hence the output of $\R(\cdot)$ reveals nothing about the local data (perfect privacy); meanwhile if $\diffp = \infty$ then the distributions of $\R(v)$ and $\R(v')$ can be arbitrarily far, so that in fact one may simply set $\R(x) = x$ and reveal local data in the clear (total lack of privacy). Thus, $\diffp\ge 0$ is typically called the {\it privacy loss} of a protocol.

There has been much previous work on private protocols for estimating the mean $\mu := \frac 1n\sum_{i=1}^n v_i$ in the LDP setting. Henceforth we assume each $v_i$ lives on the unit Euclidean sphere\footnote{One often considers the problem for the vectors being of norm at most 1, rather than exactly 1. It is easy to show that vectors $v$ in the unit ball in $\r^d$ can be mapped to $\mathbb{S}_{d} \subseteq \r^{d+1}$, simply as $(v, 1 - \|v\|_2^2)$. Thus up to changing $d$ to $d+1$, the two problems are the same. Since we are interested in the case of large $d$, we choose the version that is easier to work with.} $\mathbb S_{d-1}\subset \r^d$.
Let $\hat \mu$ be the estimate computed by the central server based on the randomized messages $\R(v_1),\ldots,\R(v_n)$ it receives. \citet{DuchiR19} showed that the asymptotically optimal expected mean squared error $\E \ltwo{\mu - \hat \mu}^2$ achievable by any one-round protocol must be at least $\Omega(\frac d{n \min (\diffp, \diffp^2)})$, which is in fact achieved by several protocols \cite{DuchiJW18,BhowmickDuFrKaRo18,ChenKO20,FeldmanTa21}. These protocols however achieved empirically different errors, with some having noticeably better constant factors than others.

Recent work of~\citet{AsiFeTa22} sought to understand the optimal 
error achievable by any protocol. Let $\R$ be any local randomizer satisfying $\diffp$-DP, and $\mathcal A$ be an aggregation algorithm for the central server such that it computes its mean estimate as $\hat \mu := \mathcal A(\R(v_1), \ldots, \R(v_n))$. Furthermore, suppose that the protocol is {\it unbiased}, so that $\E \hat \mu = \mu$ for any inputs $v_1,\ldots,v_n$. Lastly, let $\mathcal A_{\pu_\diffp}, \R_{\pu_\diffp}$ denote the $\pu_\diffp$ protocol of \cite{BhowmickDuFrKaRo18} (parameterized to satisfy $\diffp$-DP\footnote{There are multiple ways to set parameters of $\pu$ to achieve $\diffp$-DP; we assume the setting described by \citet{AsiFeTa22}, which optimizes the parameters to minimize the expected mean squared error.}). Let $\err_{n,d}(\mathcal A, \R)$ denote
$$
\sup_{v_1,\ldots,v_n\in\mathbb \mathbb{S}^{d-1}} \ltwo{\mathcal A(\R(v_1),\ldots,\R(v_n)) - \mu}^2 .
$$
\citet{AsiFeTa22} proved the remarkable theorem that for any $n, d$: 
\begin{equation*}
    \err_{n,d}(\mathcal A, \R) \ge \err_{n,d}(\mathcal A_{\pu_\diffp}, \R_{\pu_\diffp}) .
\end{equation*}

Thus, $\privunit$ is not only asymptotically optimal, but in fact actually optimal in a very strong sense (at least, amongst unbiased protocols). 

While this work thus characterizes the optimal error achievable for $\diffp$-LDP mean estimation, there are other desiderata that are important in practice. The most important amongst them are the {\it device runtime} (the time to compute $\R(v)$), the {\it server runtime} (the time to compute $\mathcal A$ on $(\R(v_1),\ldots,\R(v_n))$), and the {\it communication cost} ($\lceil \log_2 |\mathcal M|\rceil$ bits for each device to send its $\R(v)$ to the server). 
The known error-optimal algorithms (\pu \cite{BhowmickDuFrKaRo18} and \pug \cite{AsiFeTa22}) either require communicating $d$ floats or have a slower device runtime of $\Omega(e^\diffp d)$. As mean estimation is often used as a subroutine in high-dimensional learning settings, this communication cost can be prohibitive and this has led to a large body of work on reducing communication cost ~\cite{AgarwalSYKM18, girgis2020shuffled, ChenKO20, GKMM19, FeldmanTa21,ChaudhuriMoSa11}. Server runtimes of these optimal algorithms are also slow, scaling as $nd$, whereas one could hope for nearly linear time $\tilde O(n+d)$ (see \cref{tab:comp}).


\citet{ChenKO20} recently studied this tradeoff and proposed an algorithm called \sqkr, which has an optimal communication cost of $\diffp$ bits \jn{cite proof of comm lower bound?} and device runtime only $O(d \log^2 d)$.
However, this algorithm incurs error that is suboptimal by a constant factor, which can be detrimental in practice. Indeed our experiments in \cref{sec:experiments} demonstrate the significance of such constants as the utility of these algorithms does not match that of optimal algorithms even empirically, resulting for example in $10\%$ degradation in accuracy for private federated learning over MNIST with $\diffp=10$.

\citet{FeldmanTa21} give a general approach to reducing communication via rejection sampling. When applied to \pug, it yields a natural algorithm that we will call Compressed \pug. While it yields optimal error and near-optimal communication, it requires device run time that is $O(e^{\diffp} d)$. These algorithms are often used for large $d$ (e.g. in the range $10^5-10^7$) corresponding to large model sizes. The values of $\diffp$ are often in the range $4$-$12$ or more, which may be justifiable due to privacy being improved by aggregation or shuffling~\citep{Bittau17,Cheu:2019,ErlingssonFeMiRaTaTh19,FeldmanMcTa20}. For this range of values, the $\Theta(e^{\diffp} d)$ device runtime is prohibitively large and natural approaches to reduce this in~\citet{FeldmanTa21} lead to increased error. To summarize, in the high-dimensional setting, communication-efficient local randomizers are forced to choose between high device runtime or suboptimal error (see~\Cref{tab:comp}). 

Another related line of work is non-private communication efficient distributed mean estimation where numerous papers have recently studied the problem due to its importance in federated learning~\cite{SureshYuKuMc17,AlistrahGLTV17,KonevcnyRi18,AgarwalSYKM18,GKMM19,FaghriTMARR20,MT20,VargaftikBaPoGaBeYaMi21,VargaftikBaPoMeItMi22}.
Similarly to our paper, multiple works in this line of work have used random rotations to design efficient algorithms~\cite{SureshYuKuMc17,VargaftikBaPoGaBeYaMi21,VargaftikBaPoMeItMi22}. However, the purpose of these works is to develop better quantization schemes for real-valued vectors to reduce communication to $(1+o(1)) \cdot d$ bits. This is different from our goal, which is to send $k \ll d$ parameters while still obtaining the statistically optimal bounds up to a $1+o(1)$ factor. Moreover, in order to preserve privacy, our algorithms require new techniques to handle the norms of the projected vectors, and post-process them using different normalization schemes.

\subsection{Contributions}
Our main contribution is a new framework, \pjl, for private mean estimation which results in near-optimal, efficient, and low-communication algorithms. 
Our algorithm obtains the same optimal utility as \pu and \pug but with a significantly lower polylogarithmic communication complexity, and device runtime that is $O(d \log d)$ and server runtime $\tilde O(n+d)$. We also implement our algorithms and show that both the computational cost and communication cost are small empirically as well. 
\Cref{fig:err-comp} plots the error as a function of $\diffp$ for several algorithms and demonstrates the superiority of our algorithms compared to existing low-communication algorithms (see more details in~\Cref{sec:experiments}).
Moreover, we show that the optimal error bound indeed translates to fast convergence in our private federated learning simulation.

At a high level, each local randomizer in our algorithm first projects the vector to a randomly chosen lower-dimensional subspace, and then runs an optimal local randomizer in this lower-dimensional space. At first glance, this is reminiscent of the use of random projections in the Johnson-Lindenstrauss (JL) transform or the use of various embeddings in prior work (such as \citep{ChenKO20}). However, unlike the JL transform and embeddings in prior work, in our application, each point uses its own projection matrix \jn{not really true anymore}. The JL transform is designed to preserve {\em distances} between points, not the points themselves. In our application, a random projection is used to obtain a low-dimensional  unbiased estimator for a point; however the variance of this estimator is quite large (of the order of $d/k$). Our main observation is that while large, this variance is small compared to the variance added due to privacy when $k$ is chosen appropriately. This fact allows us to use the projections as a pre-processing step. With some care needed to control the norm of the projected vector which is no longer fixed, we then run a local randomizer in the lower dimensional space. We analyze the expected squared error of the whole process and show that as long as the random projection ensemble satisfies certain specific properties, the expected squared error of our algorithm is within $(1+o(1))$ factors of the optimal; the $o(1)$ term here falls with the projection dimension $k$. The required properties are easily shown to be satisfied by random orthogonal projections. We further show that more structured projection ensembles, that allow for faster projection algorithms, also satisfy the required properties, and this yields even faster device runtime. 

Although these structured projections result in fast device runtime, they still incur an expensive computational cost for the server which needs to apply the inverse transformation for each client, resulting in runtime $O(nd \log d)$. Specifically, each device sends a privated version  $\hat v_i$ of $W_i v_i$, and the server must then compute $\sum_i W^\top \hat v_i$.
To address this, we use correlated transformations in order to reduce server runtime while maintaining optimal accuracy up to $1+o(1)$ factors. In our correlated \pjl protocol, the server pre-defines a random transformation $W$, which all devices then use to define $W_i = S_i W$ where $S_i$ is a sampling matrix. The advantage is then that $\sum_i W_i^\top \hat v_i$ is replaced with $\sum_i W^\top \hat v_i = W^\top (\sum_i S_i \hat v_i)$, which can be computed more quickly as it requires only a single matrix-vector multiplication. The main challenge with correlated transformations is that the correlated client transformations result in increased variance. However, we show that the independence in choosing the sampling matrices $S_i$ is sufficient to obtain optimal error. 

Finally, we note without correlating projections each client using its own projection would imply that each projection needs to be communicated to the server. Doing this naively would require communicating $kd$ real values completely defeating the benefits of our protocol. However the projection matrix does not depend on the input and therefore can be communicated cheaply using a seed to an appropriate pseudorandom generator.




\newcommand{\vopt}{\mathsf{OPT}}

\begin{table*}[t]
\begin{center}
		\begin{tabular}{| Sc | Sc | Sc | Sc | Sc |}
		    \hline
			  & \textbf{\darkblue{Utility}} & 
           \textbf{\shortstack{\darkblue{Run-time} \\ (client)}} &
            \textbf{\shortstack{\darkblue{Run-time} \\ (server)}} &
            \textbf{\darkblue{Communication}} \\
			\hline
                \shortstack{\small Repeated \phs \\ \scriptsize{\cite{DuchiJW18,FeldmanTa21}}} & $O(\vopt)$  & $\diffp d$ & $n \ceil{\diffp}$ &  $\diffp \cdot \poly(\log d)$ \\ 
			\hline
                \shortstack{\small  \pu \\ \scriptsize{\cite{BhowmickDuFrKaRo18}}} & $\vopt$  & $ d$ & $n d$ &  $d$ \\
			\hline 
			\shortstack{\small  \sqkr \\ \scriptsize{\cite{ChenKO20}}} & $ O(\vopt)$  & $d \log^2 d$ & $n \diffp \log d + d \log^2 d$ & $\diffp \log d$ \\ 
			\hline
                \shortstack{\iftoggle{arxiv}{\small Compressed \pu}{\small Comp\pu} \\ \scriptsize{\cite{FeldmanTa21}}} & $(1+o(1)) \cdot \vopt$  & $e^\diffp d$ & $n d \log d$ & $\diffp \cdot \poly(\log d)$ \\
			\hline
                \shortstack{\small  \pfjl \\ \small{\textbf{(\Cref{sec:algs}})}} & $(1+o(1)) \cdot \vopt$  & $d \log d$ & $n d \log d $ & $\diffp \log^2 d$\\
                \hline
                \shortstack{\small \pfjl-corr \\ \small{\textbf{(\Cref{sec:algs-corr})}}} & $(1+o(1)) \cdot \vopt$  & $d \log d$ & $ n\log^3 d + n \diffp \log d +  d \log d$  & $\diffp \log^2 d$\\
			\hline
		\end{tabular}

     \caption{Comparison of Error-Runtime-Communication trade-offs for different algorithms for private mean estimation.
     The last two rows use our algorithms from~\Cref{sec:algs} and~\Cref{sec:algs-corr} with a communication budget $k \approx \diffp \log d$.
     We omit constant factors from the run-time and communication complexities.}
     \label{tab:comp}
     \end{center}
\end{table*}

\iftoggle{arxiv}{
\begin{figure}[h]
  \begin{center}
   \iftoggle{arxiv}{}{\vspace{-.5cm}}
      \begin{overpic}[width=.65\columnwidth]{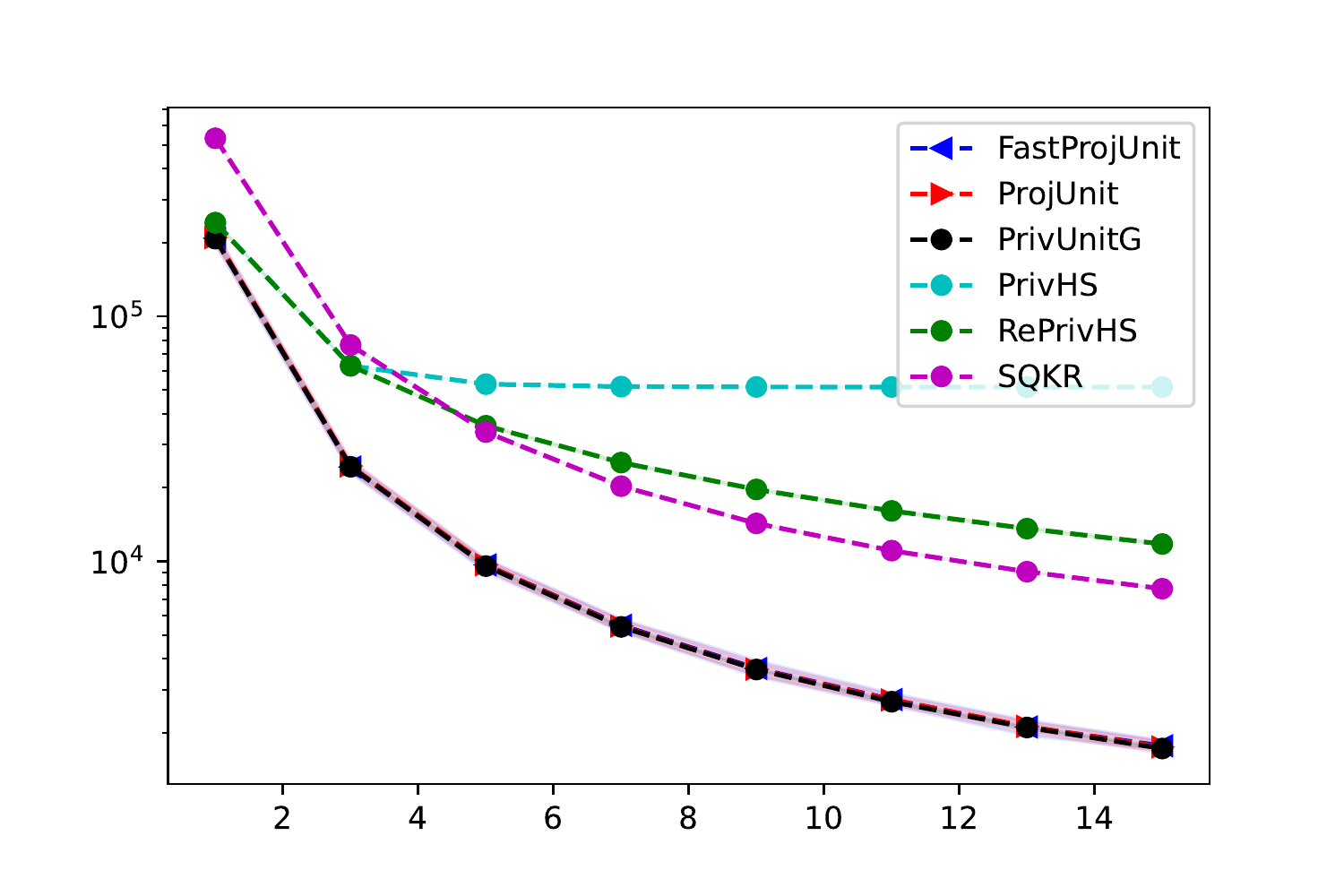}
	\put(-0.5,26){
			\rotatebox{90} {\small Squared Error}
	}
        \put(45,-3){{\small $\diffp$}}
      \end{overpic} 
    \caption{\label{fig:err-comp} Squared error of different algorithms as a function of $\diffp$ for $d = 32768$  averaged over $50$ runs with 90\% confidence intervals. The lines for the top three algorithms are almost identical.
    }
  \end{center}
\end{figure}
}
{
\begin{figure}
\floatbox[{\capbeside\thisfloatsetup{capbesideposition={right,center},capbesidewidth=4cm}}]{figure}
{\caption{\label{fig:err-comp} Squared error of different algorithms as a function of $\diffp$ for $d = 32768$  averaged over $50$ runs with 90\% confidence intervals. The lines for the top three algorithms are almost identical.
    }}
{
    \begin{overpic}[width=1.7\columnwidth]{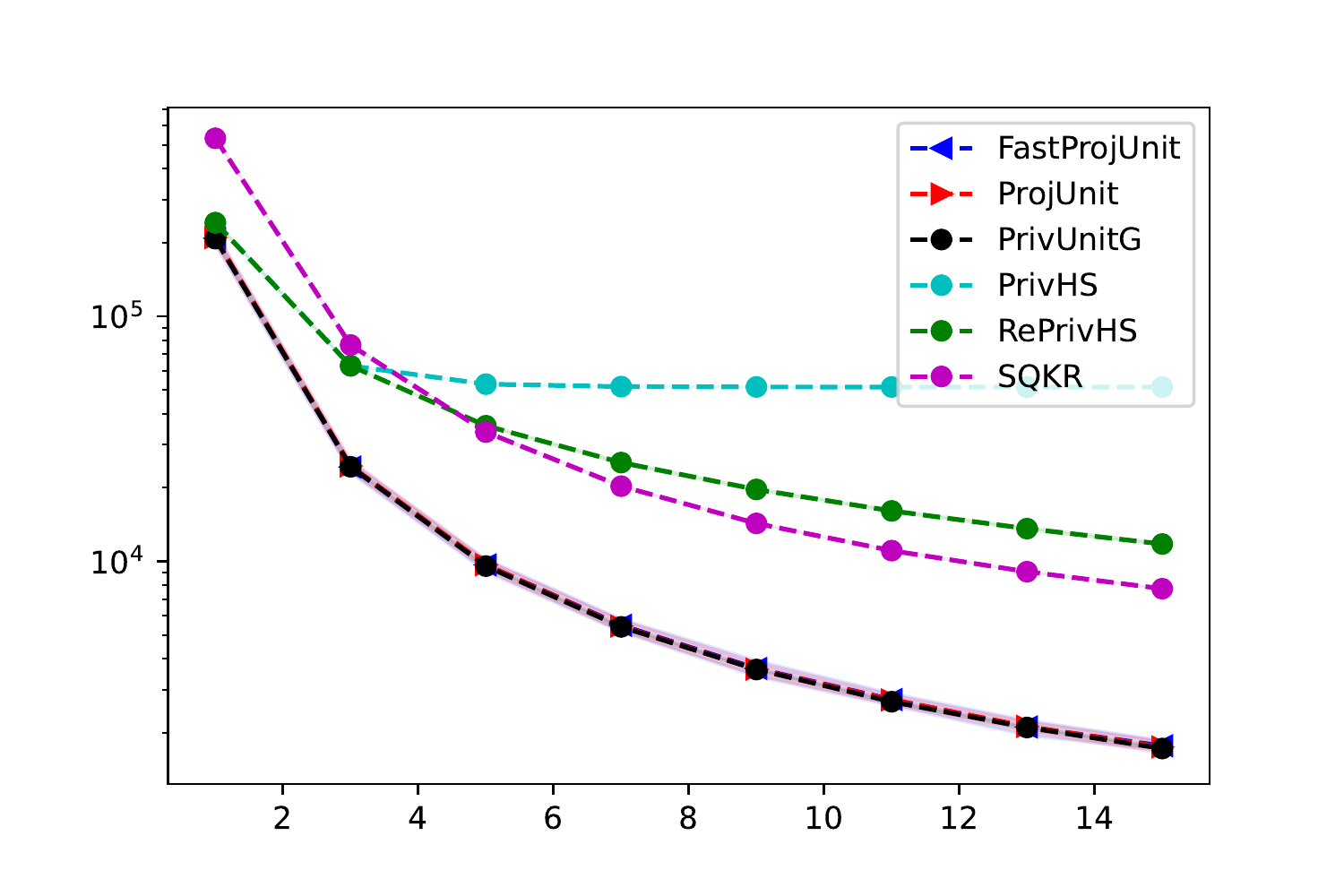}
	\put(-1,20){
			\rotatebox{90} {\small Squared Error}
	}
        \put(48,-1){{\small $\diffp$}}
      \end{overpic}
      }
\end{figure}

}





%% file: algorithms.tex
\section{A random projection framework for low-communication private algorithms}
\label{sec:algs}
In this section we propose a new algorithm, namely \pjl, which has low communication complexity and obtains near-optimal error (namely, up to a $1+o(1)$ factor of optimum). The starting point of our algorithms is a randomized projection map in $\reals^{k \times d}$ which we use to project the input vectors to a lower-dimensional space. The algorithm then normalizes the vector as a necessary pre-processing step. Finally, the local randomizer applies \pug~\cite{AsiFeTa22} (see~\Cref{alg:pug} in Appendix) over the normalized projected vector and sends the response to the server. The server then applies the inverse transformation and aggregates the responses in order to estimate the mean. We present the full details of the client and server algorithms in~\Cref{alg:privJL-cl} and~\Cref{alg:privJL-serv}.

To analyze this algorithm, we first present our general framework for an arbitrary distribution over projections. In the next sections we utilize different instances of the framework using different random projections such as random rotations and more structured transforms.  

\begin{algorithm}
	\caption{\pjl (client)}
	\label{alg:privJL-cl}
	\begin{algorithmic}[1]
		\REQUIRE Input vector $v \in \reals^d$, Distribution over projections $\mathcal{W}$.
            \STATE Randomly sample transform $\ts \in \reals^{k \times d}$ from $\mathcal{W}$
            \STATE Project the input vector $v_p = \ts v$
            \STATE Normalize $u = \frac{v_p}{\ltwo{v_p}}$
            \STATE Let $\hat u = \pug(u)$ (as in~\Cref{alg:pug})
            \STATE Send $\hat u$ and (encoding of) $\ts$ to server
	\end{algorithmic}
\end{algorithm}

\begin{algorithm}
	\caption{\pjl (server)}
	\label{alg:privJL-serv}
	\begin{algorithmic}[1]
            \STATE Receive $\hat u_1, \dots, \hat u_1$ from clients with (encodings of) transforms $\ts_1,\dots,\ts_n$
            \STATE Return the estimate 
            \iftoggle{arxiv}{
            \begin{equation*}
                \hat \mu = \frac{1}{n} \sum_{i=1}^n \ts_i^\top \hat u_i
            \end{equation*}
        }
        {
        $ \hat \mu = \frac{1}{n} \sum_{i=1}^n \ts_i^\top \hat u_i$
        }
	\end{algorithmic}
\end{algorithm}

The following theorem states the privacy and utility guarantees of \pjl for a general distribution over transformation $\mathcal{W}$ that satisfies certain properties. For ease of notation, let $\Rjl$ denote the \pjl local randomizer of the client (\Cref{alg:privJL-cl}), and $\Ajl$ denote the server aggregation of \pjl (\Cref{alg:privJL-serv}). 
To simplify notation, we let \begin{equation*}
   \err_{n,d}(\pug) 
   = \err_{n,d}(\mathcal A_{\pug_\diffp}, \R_{\pug_\diffp})
   = c_{d,\diffp} \frac{d}{n \diffp}
\end{equation*}
denote the error of the \pug $\diffp$-DP protocol where $A_{\pug_\diffp}$, $\R_{\pug_\diffp}$ denote the \pug protocol with optimized parameters (see~\Cref{alg:pug}) and $c_{d,\diffp} = O(1)$ is a constant~\cite{AsiFeTa22}.
\iftoggle{arxiv}{}{We defer the proof to~\Cref{sec:proof-thm-jl-gen}.}
\begin{theorem}
\label{thm:err-jl-gen}
    Let $k \le d$ and assume that the transformations $\ts_i \in \reals^{k \times d}$ are independently chosen from a distribution $\mathcal{W}$ that satisfies:
    \begin{enumerate}
        \item Bounded operator norm: $\E\left[\norm{\ts_i^\top}^2\right] \le {d/k} + \beta_{\mathcal{W}}$.
        \item Bounded bias:  $\ltwo{\E\left[ \frac{\ts_i^\top\ts_i v}{\ltwo{\ts_i v}}\right] - v} \le \sqrt{\alpha_{\mathcal{W}}}$ for all unit vectors $v \in \reals^d$. 
    \end{enumerate}
    Then for all unit vectors $v_1,\ldots,v_n \in \reals^d$, setting $\hat \mu = \Ajl \left( \Rjl(v_1), \dots, \Rjl(v_n) \right)$, the local randomizers $\Rjl$ are $\diffp$-DP and  
\begin{align*}
        \E\left[\ltwo{  \hat \mu - \frac{1}{n}\sum_{i=1}^n v_i  }^2 \right] 
        \le \err_{n,d}(\pug) \cdot \left(1 + \frac{k\beta_{\mathcal{W}}}{d} + O\left( \frac{\diffp + \log k}{k} \right) \right)
         + \alpha_{\mathcal{W}}.
    \end{align*}
\end{theorem}
\knote{As stated, the theorems don't say anything about the estimators being closed to unbiased. Is there some statement we can make?}
\iftoggle{arxiv}{
\begin{proof}
    First, note that the claim about privacy follows directly from the privacy guarantees of \pug~\cite{AsiFeTa22} as our algorithm applies \pug over a certain input vector with unit norm. \\
    For accuracy, note that $\hat \mu =  \frac{1}{n}\sum_{i=1}^n \ts_i^\top \hat{u}_i$, therefore
    \begin{align*}
    \E\left[\ltwo{\hat \mu - \frac{1}{n}\sum_{i=1}^n v_i}^2\right] 
        & =  \E\left[\ltwo{\frac{1}{n}\sum_{i=1}^n \ts_i^\top \hat{u}_i - v_i}^2\right] \\
        & =  \E\left[\ltwo{\frac{1}{n}\sum_{i=1}^n \ts_i^\top \hat{u}_i - \ts_i^\top u_i + \ts_i^\top u_i - v_i}^2\right] \\
        & \stackrel{(i)}{=}   \frac{1}{n^2}\sum_{i=1}^n \E\left[\ltwo{\ts_i^\top \hat{u}_i - \ts_i^\top u_i}^2\right] + \frac{1}{n^2} \E\left[ \ltwo{\sum_{i=1}^n \ts_i^\top u_i - v_i}^2\right] \\
        & \le \frac{1}{n} \max_{i \in [n]}\E\left[\ltwo{\ts_i^\top}^2\right] \cdot \err_{1,k}(\pug) + \frac{1}{n^2} \E\left[ \ltwo{\sum_{i=1}^n \ts_i^\top u_i - v_i}^2\right].
    \end{align*}
\noindent where $(i)$ follows since $\E[\hat u] = u$ as $\pug$ is unbiased. Now we analyze each of these two terms separately. For the first term, as $\E[\norm{\ts_i}^2] \le d/k + \beta_{\mathcal{W}}$ for all $i \in [n]$ we have that is is bounded by
    \begin{align*}
    \max_{i \in [n]} \E\left[\ltwo{\ts^\top}^2\right] \cdot \err_{1,k}(\pug)
        & \le  \left(\frac{d}{k}+\beta_{\mathcal{W}}\right) c_{k,\diffp} \frac{k}{\diffp} \\
        & = \left(\frac{d}{\diffp}+\frac{\beta_{\mathcal{W}} k}{\diffp}\right) c_{d,\diffp} \frac{c_{k,\diffp}}{c_{d,\diffp}} \\
        & = \left(\frac{d}{\diffp}+\frac{\beta_{\mathcal{W}} k}{\diffp}\right) c_{d,\diffp} \cdot \left(1 + O\left( \frac{\diffp + \log k}{k} \right) \right) \\
        & = \err_{1,d}(\pug) \cdot \left(1 + \frac{\beta_{\mathcal{W}} k}{d} + O\left( \frac{\diffp + \log k}{k} \right) \right),
    \end{align*}
    where the third step follows from~\Cref{prop:C_eps}.
    For the second term,
    \begin{align*}
    \E\left[ \ltwo{\sum_{i=1}^n \ts_i^\top u_i - v_i}^2\right]
        &=    \sum_{i=1}^n\sum_{j\ne i}\E\left[ \left\langle\ts_i^\top u_i - v_i, \ts_j^\top u_j - v_j\right\rangle\right] + \sum_{i=1}^n\E\left[ \ltwo{\ts_i^\top u_i - v_i}^2\right]\\
        &\le \sum_{i=1}^n\sum_{j\ne i} \ltwo{ \E\ts_i^\top u_i - v_i}\cdot \ltwo{ \E\ts_j^\top u_j - v_j} + \sum_{i=1}^n\E\left[ \ltwo{\ts_i^\top u_i - v_i}^2\right]\\
        &\le n(n-1)\alpha_{\mathcal{W}} + \sum_{i=1}^n\E\left[ \ltwo{\ts_i^\top u_i}^2 +\ltwo{v_i}^2 - 2v_i^\top W_i^\top  u_i\right]\\
        &= n(n-1)\alpha_{\mathcal{W}} + \sum_{i=1}^n\E\left[ \ltwo{\ts_i^\top u_i}^2 + 1 - 2\ltwo{W_i v_i}\right]\\
        &\le n(n-1)\alpha_{\mathcal{W}} + n \max_{i \in [n]} \E\left[ \ltwo{\ts_i^\top}^2\right] + n.
    \end{align*} 
    Overall, this shows that 
    \begin{align*}
        \E\left[\ltwo{ \hat \mu- \frac{1}{n}\sum_{i=1}^n v_i  }^2 \right] 
        & \le \err(\pug_d,n) \cdot \left(1 + O\left( \frac{\diffp + \log k}{k} \right) \right) \\
        & \quad     + O\left(\frac{d}{nk}\right) + \frac{1}{n} + \frac{(n-1)\alpha_{\mathcal{W}}}{n}.
    \end{align*} 
    Noticing that $\err_{n,d}(\pug) = c_{d,\diffp} \cdot \frac{d}{n\diffp}$ for some constant $c_{d,\diffp}$ (see~\cite{AsiFeTa22}), this implies that 
    \begin{align*}
        \E\left[\ltwo{  \hat \mu- \frac{1}{n}\sum_{i=1}^n v_i  }^2 \right] 
        & \le \err_{n,d}(\pug) \cdot \left(1 + O\left( \frac{\diffp + \log k}{k} \right) \right)    + \alpha_{\mathcal{W}}.
    \end{align*}
This completes the proof.
\end{proof}
}

\subsection{\pjl~using Random Rotations}
\label{sec:algs-jl}
Building on the randomized projection framework of the previous section, in this section we instantiate it with a random rotation matrix. In particular, we sample $\ts \in \reals^{k \times d}$ with the structure
\begin{equation}
\label{eq:ts-rot}
    \tsH = \sqrt{\frac{d}{k}} S U,
\end{equation}
where $U \in \reals^{d \times d}$ is a random rotation matrix such that $U^\top U = I$ , and $S \in \reals^{k \times d}$ is a sampling matrix where each row has a single $1$ in a uniformly random location (without repetitions). 
The following theorem states our guarantees for this distribution.
\begin{theorem}
\label{thm:err-rand-rot}
    Let $k \le d$ and $\ts \in \reals^{k \times d}$ be a random rotation matrix sampled as described in~\eqref{eq:ts-rot}.
    Then for all unit vectors $v_1,\ldots,v_n \in \reals^d$, setting $\hat \mu =  \Ajl \left( \Rjl(v_1), \dots, \Rjl(v_n) \right) $, the local randomizers $\Rjl$ are $\diffp$-DP and 
    \begin{align*}
    \E\left[\ltwo{ \hat \mu - \frac{1}{n}\sum_{i=1}^n v_i  }^2 \right] 
    \le  \err_{n,d}(\pug) \cdot  \left(1 +   O\left(  \frac{\diffp + \log k}{k} \right) \right)
     + O \left(\frac{1}{k^2}\right).
    \end{align*}
\end{theorem}

The proof follows directly from~\Cref{thm:err-jl-gen} and the following proposition which proves certain properties of random rotations.
\iftoggle{arxiv}{}{We defer the proof to~\Cref{sec:proof-prop-rotation}}.
\begin{proposition}
\label{prop:rotation-prop}
    Let $\ts \in \reals^{k \times d}$ be a random rotation matrix sampled as described in~\eqref{eq:ts-rot}. Then
    \begin{enumerate}
        \item Bounded operator norm: 
        \iftoggle{arxiv}{
        \begin{equation*}
            \|W^\top\|\le \sqrt{\frac{d}{k}} .
        \end{equation*}
        }
        {
       $ \|W^\top\|\le \sqrt{\frac{d}{k}} .$
        }
        
        \item Bounded bias: for every unit vector $v \in \reals^d $:
        \iftoggle{arxiv}{
        \begin{equation*}
            \left\| \E \frac{\ts^\top \ts v}{\|\ts v\|} - v \right\| = O(1/k). 
        \end{equation*}
        }
        {
        $ \left\| \E \frac{\ts^\top \ts v}{\|\ts v\|} - v \right\| = O(1/k).$
        }
    \end{enumerate}
\end{proposition}
\iftoggle{arxiv}{
\begin{proof}
The first item follows immediately as $U \in \reals^{d \times d}$ is a random rotation matrix where $U^\top U = I$, hence $\norm{U} \le 1$. 

For the second item, we use a change of variables. Let $\ts' = \ts P^\top$ where $P$ is the rotation matrix such that $P v = e_1$, the first standard basis vector. Recall that the rotation matrix $P$ is orthogonal i.e. $P^\top = P^{-1}$. Due to the rotational symmetry of rotation matrices, $\ts'$ is a random also a random rotation matrix. Note that $\ts = \ts' P$.
\begin{align*}
    \E_{\ts}\left[\frac{\ts^\top \ts v}{\ltwo{\ts v}} \right] &= \E_{{\ts'}}\left[\frac{1}{\ltwo{\ts' P v}}P^\top{\ts'}^\top \ts' P v\right]\\
        &=P^\top\E\left[\underbrace{\frac{1}{\ltwo{\ts' e_1}} {\ts'}^\top \ts' e_1}_z\right]
\end{align*}
Notice that $z_j = \frac{1}{\ltwo{\ts' e_1}} e_j^\top {\ts'}^\top \ts' e_1 = \langle \ts' e_j, \frac{1}{\ltwo{\ts' e_1}} \ts' e_1\rangle$. 
First, note that $\E[z_j]=0$ for all $j>1$. Moreover, $z_1 = \ltwo{\ts' e_1}$, therefore
because $\ts'$ is a random rotation matrix, \Cref{lemma:unit-proj} implies that $\E[z_1] = \sqrt{d/k}(\sqrt{k/d} + O(1/\sqrt{kd})) = 1 + O(1/k)$.
We let $c = \E[z_1]$.
Thus, $\E[z] = ce_1$ and $\E[\ts^\top \ts v /\ltwo{\ts v}] = c P^\top e_1 = c P^\top P v = c v$. Therefore, $\ltwo{\E[\frac{\ts^\top \ts v}{\ltwo{\ts v}} - v]} = |c-1| \ltwo{v} = O(1/k)$.
\end{proof}
}

We also have similar analysis for Gaussian transforms with an additional $O(\sqrt{k/d})$ factor in the first term. We include the analysis in~\Cref{sec:gauss}.

\subsection{Fast \pjl~using the SRHT}
\label{sec:algs-fjl}

While the random rotation based randomizer enjoys near-optimal error and low communication complexity, its runtime complexity is somewhat unsatisfactory as it requires calculating $\ts v$ for $\ts \in \reals^{k \times d}$, taking time $O(kd)$.
In this section, we propose a \pjl algorithm using the Subsampled Randomized Hadamard transform (SRHT), which is closely related to the fast JL transform~\cite{AilonCh09}. We show that this algorithm has the same optimality and low-communication guarantees as the random rotations version, and additionally has an efficient implementation that take $O(d \log d)$ client runtime for the transformation.

The SRHT ensemble contains matrices $\tsH \in \reals^{k \times d}$ with the following structure:
\begin{equation}
\label{eq:ts-had}
    \tsH = \sqrt{\frac{d}{k}} S H D,
\end{equation}
where $S \in \reals^{k \times d}$ is a sampling matrix where each row has a single $1$ in a uniformly random location sampled without replacement, $H \in \reals^{d \times d}$ is the Hadamard matrix, and $D \in \reals^{d \times d}$ is a diagonal matrix where $D_{ii}$ are independent samples from the Rademacher distribution, that is, $D_{ii} \sim \mathsf{Unif}\{-1, +1\}$.
The main advantage of the SRHT is that there exist efficient algorithms for matrix-vector multiplication with $H$.

The following theorem presents our main guarantees for the SRHT-based \pjl algorithm. 
\begin{theorem}
\label{thm:err-fjl}
    Let $k \le d$ and $\ts$ be sampled from the SRHT ensemble as described in~\eqref{eq:ts-had}.  Then for all unit vectors $v_1,\ldots,v_n \in \reals^d$, setting $\hat \mu =  \Ajl \left( \Rjl(v_1), \dots, \Rjl(v_n) \right) $,  the local randomizers $\Rjl$ are $\diffp$-DP and 
    \begin{align*}
    \E\left[\ltwo{ \hat \mu - \frac{1}{n}\sum_{i=1}^n v_i  }^2 \right] 
    &\le  \err_{n,d}(\pug)  \cdot \left(1 + O\left( \frac{\diffp + \log k}{k} \right) \right) + O \left({\frac{\log^2 d}{k}} \right).
    \end{align*}

\end{theorem}
\begin{remark}\label{remark:cc-pfjl}
    The communication complexity of SRHT-based \pjl can be reduced to $O(k\log d + \log^2d)$. To see this, note that $\hat u$ is a $k$-dimensional vector. Moreover, the matrix $\ts = \sqrt{d/k} \cdot S H D$ can be sent in $O(k \log d)$ as follows: $S$ has $k$ rows, each with a single entry that contains $1$, hence we can send the indices for each row using $\log d$ bits for each row. Moreover, $H$ is the Hadamard transform and need not be sent. Finally, $D$ is a diagonal matrix with entries $D_{ii} \sim \mathsf{Unif}\{-1, +1\} $. By standard techniques~\cite{SchmidtSiSr95,AlonSp00}, we only need the entries of $D$ to be $O(\log(d))$-wise independent for \cref{prop:HD-prop} to hold. Thus $O(\log^2 d)$ bits suffice to communicate a sampled $D$.
\end{remark}

The proof of the theorem builds directly on the following two properties of the SHRT. 
\begin{proposition}
\label{prop:HD-prop}
    Let $\ts$ be sampled from the SRHT ensemble as described in~\eqref{eq:ts-had}.
    Then we have
    \begin{enumerate}
        \item Bounded operator norm: 
        \iftoggle{arxiv}{
        \begin{equation*}
            \E \left[\norm{\ts^\top} \right] = \E \left[\norm{\ts} \right] \le \sqrt{d/k} 
        \end{equation*}
        }
        {
        $   \E \left[\norm{\ts^\top} \right] = \E \left[\norm{\ts} \right] \le \sqrt{d/k}. $
        }
        \item Bounded bias: for every unit vector $v \in \reals^d $:
        \iftoggle{arxiv}{
        \begin{equation*}
            \left\| \E \frac{\ts^\top \ts v}{\|\ts v\|} - v \right\| = O(\log(d)/\sqrt{k}) 
        \end{equation*}
        }
        {
            $\left\| \E \frac{\ts^\top \ts v}{\|\ts v\|} - v \right\| = O(\log(d)/\sqrt{k}) .$
        }
    \end{enumerate}
\end{proposition}
\Cref{thm:err-fjl} now follows directly from the bounds of~\Cref{thm:err-jl-gen} using $\alpha_{\mathcal{W}} = O(\log^2 (d)/k)$.
\iftoggle{arxiv}{Now we prove~\Cref{prop:HD-prop}.}{We defer the proof to~\Cref{proof:prop-HD-prop}}.

\iftoggle{arxiv}{
\begin{proof}
The bound on the operator norm is straightforward and follows from the fact that the Hadamard transform has operator norm bounded by $1$.

Next we bound the bias. Let $\delta=\min(1/d^2,k/2d)$ and let $E_1$ denote the event where $\|\ts v\| \in 1\pm O(\ln(k/\delta)/\sqrt{k})$. By  \cref{cor:cnw}, $E_1$ happens with probability $1-\delta$. Note that $\ts^\top \ts$ is PSD and $\E[\ts^\top \ts] = I$.
Thus,
\begin{align*}
	\left\| \E \left[\frac{\ts^\top \ts v}{\|\ts v\|} - \ts^\top \ts v \right]\right\| 
        &\le \left\| \E \left[\frac{\ts^\top \ts v}{\|\ts v\|} - \ts^\top \ts v \vert E_1\right]\right\| + \left\| \E \left[\frac{\ts^\top \ts v}{\|\ts v\|} - \ts^\top \ts v \vert \overline{E_1}\right]\right\| \P(\overline{E_1})\\
	&\le \left\| \E \left[\left(\frac{1}{\|\ts v\|}-1\right)\ts^\top \ts v \vert E_1\right]\right\| + (\|\ts^\top\|+ 1) \P(\overline{E_1})\\
		&\le \left\| \E \left[\left(\frac{1}{\|\ts v\|}-1\right)\ts^\top \ts \vert E_1\right]\right\| + (\|\ts^\top\|+ 1) \P(\overline{E_1})\\		
		&\le \left\| \E \left[\left|\frac{1}{\|\ts v\|}-1\right|\ts^\top \ts \vert E_1\right]\right\| + (\|\ts^\top\|+ 1) \P(\overline{E_1})\\
		&\le O(\ln(k/\delta)/\sqrt{k}) + (\sqrt{d/k}+1)\delta
\end{align*}
Substituting in the value of $\delta$ gives the desired bound on the bias, by noticing that $\ltwo{\E[\ts^\top \ts \mid E_1]} \le 2$ since for any unit vector $x$,
\begin{align*}
    x^\top I x
    & = \E[x^\top \ts^\top \ts x] \\
    & = \E[x^\top\ts^\top \ts x\mid E_1] \P(E_1) + \E[x^\top\ts^\top \ts x\mid \overline{E_1}] \P(\overline{E_1}) \\
    & \ge \E[x \ts^\top \ts x \mid E_1] (1-\delta).
\end{align*}
In other words, $\E[x \ts^\top \ts x \mid E_1] \le \frac{1}{1-\delta}$ for all $x \in \reals^d$ with unit norm.
\end{proof}
}

\begin{remark}
    While our randomizers in this section pay an additive term that does not decrease with $n$ (e.g. $\log^2(d)/k$ for SRHT), this term is negligible in most settings of interest. Indeed, using~\Cref{thm:err-fjl} and the fact that $\err_{n,d}(\pug) = c_{d,\diffp} d/n\diffp$, we get that the final error of our SRHT algorithm is roughly $c_{d,\diffp} d/n\diffp (1+o(1)) + O(\log^2(d)/k)$. This implies that in the high-dimensional setting the bias term is negligible.

    However, to cover the whole spectrum of parameters, we develop a nearly unbiased versions of these algorithms in~\Cref{sec:algs-unb}. In particular, we show in~\Cref{thm:err-unb-fjl} that our unbiased version has error 
    \begin{equation*}
       \err_{n,d}(\pug) \cdot \left(1 + O\left( \frac{\diffp + \log k}{k} + \sqrt{\frac{\log(nd/k)}{k}} \right) \right).
    \end{equation*}
\end{remark}

%% file: algs-fixed-transformation.tex
\section{Efficient Server Runtime via Correlated Sampling}
\label{sec:algs-corr}
One downside of the algorithms in the previous section is that server runtime can be costly: indeed, the server has to apply the inverse transformation (matrix multiplication) for each client, resulting in runtime $O(nd\log d)$. In this section, we propose a new protocol that significantly reduces server runtime to $O(n\log^3 d + d\log d+ nk)$ while retaining similar error guarantees. The protocol uses correlated transformations between users which allows the server to apply matrix multiplication only a small number of times. However, clients' transformations cannot be completely correlated to allow to control the bias.

The protocol works as follows: 
the server samples 
$U \in \reals^{d \times d}$ 
from the Randomized Hadamard transform: $U = H D$ where $H \in \reals^{d \times d}$ is the Hadamard transform,  
and $D \in \reals^{d \times d}$ is a diagonal matrix where each diagonal entry is independently sampled from the Rademacher distribution.
Then, client $i \in [n]$, samples a random sampling matrix $S_i \in \reals^{k \times d}$, and uses $U$ to define the transform $\ts_i \in \reals^{k \times d}$:
\begin{equation}
    \label{eq:shared-ts}
    \ts_i = \sqrt{\frac{d}{k}} S_i U.
\end{equation}
We describe the full details of the client and server algorithms for correlated \pjl in~\Cref{alg:privJL-cl-corr} and~\Cref{alg:privJL-serv-corr}, and denote them $\Rcjl$ and $\Acjl$, respectively.
We have the following guarantee.
\iftoggle{arxiv}{}{
We defer the proof to~\Cref{sec:proof-thm-corr}.
}

\begin{algorithm}
	\caption{Correlated \pjl (client)}
	\label{alg:privJL-cl-corr}
	\begin{algorithmic}[1]
		\REQUIRE Input vector $v \in \reals^d$.
            \STATE Randomly sample diagonal $D$ from the Rademacher distribution based on predefined seed 
            \STATE Sample $S \in \reals^{k \times d}$ where each row is chosen uniformly at random without replacement from standard basis vectors $\{e_1,\dots,e_d \}$
            \STATE Project the input vector $v_p = S H D v$ where $H \in \reals^{d \times d}$ is the Hadamard matrix
            \STATE Normalize $u = \frac{v_p}{\ltwo{v_p}}$
            \STATE Let $\hat u = \pug(u)$ (as in~\Cref{alg:pug})
            \STATE Send $\hat u$, and (an encoding of) $S$ to server
	\end{algorithmic}
\end{algorithm}

\begin{algorithm}
	\caption{Correlated \pjl (server)}
	\label{alg:privJL-serv-corr}
	\begin{algorithmic}[1]
            \STATE Receive $\hat u_1, \dots, \hat u_1$ from clients with (encodings of) transforms $S_1,\dots,S_n$
            \STATE Sample diagonal matrices $D$ from  Rademacher distribution based on predefined seed 
            \STATE Let $U = H D$ where $H \in \reals^{d \times d}$ is the Hadamard matrix
            \STATE Return the estimate
            \iftoggle{arxiv}{
            \begin{equation*}
                \hat \mu = \frac{1}{n}  U^\top
            \sum_{i =1}^n S_i^\top \hat u_i
            \end{equation*}
            }
            {
            $\hat \mu = \frac{1}{n}  U^\top
            \sum_{i=1}^{n} S_i^\top \hat u_i$
            }
	\end{algorithmic}
\end{algorithm}

\begin{theorem}
\label{thm:err-fjl-fixed}
    Let $k \le d$.  Then for all unit vectors $v_1,\ldots,v_n \in \reals^d$,  setting $\hat \mu =  \Acjl \left( \Rcjl(v_1), \dots, \Rcjl(v_n) \right) $,  the local randomizers $\Rcjl$ are $\diffp$-DP and
    \begin{align*}
    \E\left[\ltwo{ \hat \mu - \frac{1}{n}\sum_{i=1}^n v_i  }^2 \right]
    \le  \err_{n,d}(\pug)  \left(1 + O\left( \frac{\diffp + \log k}{k} \right) \right) + O \left({\frac{\log^2 d}{k}} \right).
    \end{align*}
    Moreover, server runtime is $O(n  \log(d) \log^2 (nd) + d \log d + n k )$.
\end{theorem}
\iftoggle{arxiv}{
The proof of this result follows from the next proposition. 
\begin{proposition}
\label{thm:err-fixed-hd}
    Let $k\le d$, $U$ and $\ts_1,\dots,\ts_n$ be sampled as described in~\eqref{eq:shared-ts}.
    Moreover,  $U$  and $\ts_i$ for $i \in [n]$ satisfy:
    \begin{enumerate}
        \item Bounded operator norm: $\norm{U^\top} \le 1$.
        \item Bounded bias:  $\ltwo{\E\left[ \frac{\ts_i^\top\ts_i v}{\ltwo{\ts_i v}}\right] - v} \le \sqrt{\alpha_{\mathcal{W}}}$ for all unit vectors $v \in \reals^d$.
    \end{enumerate}
    Then for all unit vectors $v_1,\ldots,v_n \in \reals^d$, 
    setting $\hat \mu = \Acjl \left( \Rcjl(v_1), \dots, \Rcjl(v_n) \right)$, 
    \begin{align*}
        \E\left[\ltwo{  \hat \mu - \frac{1}{n}\sum_{i=1}^n v_i  }^2 \right] 
        \le \err_{n,d}(\pug) \cdot \left(1 + O\left( \frac{\diffp + \log k}{k}  + \frac{n \diffp \log^2 (nd)}{d k} \right)  \right) + \alpha_{\mathcal{W}}  .
    \end{align*}
\end{proposition}
Before proving the proposition, we can now prove~\Cref{thm:err-fjl-fixed}.
\begin{proof}
    The proof follows from~\Cref{thm:err-fixed-hd} by noting that the server returns $\sum_{i=1}^n \ts_i^\top \hat u_i/n$. The first property holds immediately from the definition of $U_j$. Moreover, for the second property, \Cref{prop:HD-prop} implies that $\alpha_{\mathcal{W}} = O(\log^2(d)/k)$. Since $\err_{n,d}(\pug) = \Theta(d/n\diffp)$, the claim about utility follows.\\
    Now we prove the part regarding runtime. First, note that calculating the matrix $D$ can be done efficiently using standard techniques~\cite{SchmidtSiSr95,AlonSp00}.
    The server has to calculate the quantity
        \begin{align*}
         U^\top
            \sum_{i=1}^n S_i^\top \hat u_i.
    \end{align*}  
    Note that the summation has vectors which are $k$-sparse, therefore can be done in time $O(nk)$. Then, we have a multiplication step by Hadamard transform,which can be done in $O(d \log d)$. 
\end{proof}

We now prove~\Cref{thm:err-fixed-hd}.
\begin{proof}
    Note that $\hat \mu =  \frac{1}{n}\sum_{i=1}^n \ts_i^\top \hat{u}_i$. Therefore we have
    \begin{align*}
    \E\left[\ltwo{\hat \mu - \frac{1}{n}\sum_{i=1}^n v_i}^2\right] 
        & =  \E\left[\ltwo{\frac{1}{n}\sum_{i=1}^n \ts_i^\top \hat{u}_i - v_i}^2\right] \\
        & =  \E\left[\ltwo{\frac{1}{n}\sum_{i=1}^n \ts_i^\top \hat{u}_i - \ts_i^\top u_i + \ts_i^\top u_i - v_i}^2\right] \\
        & \stackrel{(i)}{=}   \frac{1}{n^2}\sum_{i=1}^n \E\left[\ltwo{\ts_i^\top \hat{u}_i - \ts_i^\top u_i}^2\right] + \frac{1}{n^2} \E\left[ \ltwo{\sum_{i=1}^n \ts_i^\top u_i - v_i}^2\right] \\
        & \le \frac{1}{n} \max_{i \in [n]}\E\left[\ltwo{\ts_i^\top}^2\right] \cdot \err_{1,k}(\pug) + \frac{1}{n^2} \E\left[ \ltwo{\sum_{i=1}^n \ts_i^\top u_i - v_i}^2\right].
    \end{align*} 
    where $(i)$ follows since $\E[\hat u] = u$ as $\pug$ is unbiased. Now we analyze each of these two terms separately. For the first term, as $\E[\norm{\ts_i}^2] \le d/k$ for all $i \in [n]$ we have that it is bounded by
    \begin{align*}
    \max_{i \in [n]} \E \left[\ltwo{\ts_i^\top}^2 \right] \cdot \err_{1,k}(\pug)
        \le  \frac{d}{k} c_{k,\diffp} \frac{k}{\diffp} 
        & = \frac{d}{\diffp} c_{d,\diffp} \frac{c_{k,\diffp}}{c_{d,\diffp}} \\
        & = \frac{d}{\diffp} c_{d,\diffp} \cdot \left(1 + O\left( \frac{\diffp + \log k}{k} \right) \right) \\
        & = \err_{1,d}(\pug) \cdot \left(1 + O\left( \frac{\diffp + \log k}{k} \right) \right),
    \end{align*}
    where the third step follows from~\Cref{prop:C_eps}.
    For the second term,
    \begin{align*}
    \E\left[ \ltwo{\sum_{i=1}^n \ts_i^\top u_i - v_i}^2\right]
        &=    \sum_{i=1}^n\sum_{j\ne i}\E\left[ \left\langle\ts_i^\top u_i - v_i, \ts_j^\top u_j - v_j\right\rangle\right] + \sum_{i=1}^n\E\left[ \ltwo{\ts_i^\top u_i - v_i}^2\right].
    \end{align*} 
    For the second term note that
    \begin{align*}
     \sum_{i=1}^n\E\left[ \ltwo{\ts_i^\top u_i - v_i}^2\right]
        &=  \sum_{i=1}^n\E\left[ \ltwo{\ts_i^\top u_i}^2 +\ltwo{v_i}^2 - 2v_i^\top W_i^\top  u_i\right]\\
        &=  \sum_{i=1}^n\E\left[ \ltwo{\ts_i^\top u_i}^2 + 1 - 2\ltwo{W_i v_i}\right]\\
        &\le n \max_{i \in [n]} \E\left[ \ltwo{\ts_i^\top}^2\right] + n \le n(d/k + 1).
    \end{align*} 
For the first term, we have
\begin{align*}
\E&\left[\left\langle W_{i}^{\top}u_{i}-v_{i},W_{j}^{\top}u_{j}-v_{j}\right\rangle \right]\\
&=\E\left[\left\langle W_{i}^{\top}\left(u_{i}-Wv_{i}\right)+W_{i}^{\top}W_{i}v_{i}-v_{i},W_{j}^{\top}\left(u_{j}-W_{j}v_{j}\right)+W_{j}^{\top}W_{j}v_{j}-v_{j}\right\rangle\right]\\
&=\E\left[\left\langle W_{i}^{\top}\left(u_{i}-Wv_{i}\right),W_{j}^{\top}\left(u_{j}-W_{j}v_{j}\right)\right\rangle  \right]+\E\left[\left\langle W_{i}^{\top}W_{i}v_{i}-v_{i},W_{j}^{\top}\left(u_{j}-W_{j}v_{j}\right)\right\rangle \right]\\
&\qquad+\E\left[\left\langle W_{i}^{\top}\left(u_{i}-Wv_{i}\right),W_{j}^{\top}W_{j}v_{j}-v_{j}\right\rangle  \right]+\E\left[\left\langle W_{i}^{\top}W_{i}v_{i}-v_{i},W_{j}^{\top}W_{j}v_{j}-v_{j}\right\rangle \right]
\end{align*}

Because $\E_{S_{i}}\left[\frac{d}{k}S_{i}^{\top}S_{i}\right]=I$, $H^{\top}H=I$,
and $D^{\top}D=I$, we can evaluate the second term:

\[
\E_{S_{i}}\left[\left\langle W_{i}^{\top}W_{i}v_{i}-v_{i},W_{j}^{\top}\left(u_{j}-W_{j}v_{j}\right)\right\rangle \right]=\E_{S_{i}}\left[\left\langle \frac{d}{k}D^{\top}H^{\top}S_{i}^{\top}S_{i}HDv_{i}-v_{i},W_{j}^{\top}\left(u_{j}-W_{j}v_{j}\right)\right\rangle \right]=0
\]

Similarly, we can evaluate the fourth term:

\begin{align*}
\E\left[\left\langle W_{i}^{\top}W_{i}v_{i}-v_{i},W_{j}^{\top}W_{j}v_{j}-v_{j}\right\rangle \right] & =\E\left[\left\langle \frac{d}{k}D^{\top}H^{\top}S_{i}^{\top}S_{i}HDv_{i}-v_{i},W_{j}^{\top}W_{j}v_{j}-v_{j}\right\rangle \right]=0
\end{align*}

The third term is similar. Thus, we only need to bound the first term.
First we give an upper bound that holds with probability 1.

\begin{align*}
\left\langle W_{i}^{\top}\left(u_{i}-W_{i}v_{i}\right),W_{j}^{\top}\left(u_{j}-W_{j}v_{j}\right)\right\rangle  & \le\left\Vert W_{i}^{\top}\right\Vert \left\Vert W_{j}^{\top}\right\Vert \left(\left\Vert W_{i}\right\Vert +1\right)\left(\left\Vert W_{j}\right\Vert +1\right)\\
 & \le O\left(d^{2}/k^{2}\right)
\end{align*}

Let $E_{1}$ be the event that $\left\Vert W_{i}v_{i}\right\Vert,\left\Vert W_{j}v_{j}\right\Vert \in1\pm O\left(\ln\left(k/\delta\right)/\sqrt{k}\right)$ where $\delta$ is a parameter to be chosen later. We will split the expectation depending on the event $E_{1}$.

\begin{align*}
&\left\langle W_{i}^{\top}\left(u_{i}-W_{i}v_{i}\right),W_{j}^{\top}\left(u_{j}-W_{j}v_{j}\right)\right\rangle \\
&=\left(\frac{1}{\left\Vert W_{i}v_{i}\right\Vert }-1\right)\left(\frac{1}{\left\Vert W_{j}v_{j}\right\Vert }-1\right)v_{i}^{\top}W_{i}^{\top}W_{i}W_{j}^{\top}W_{j}v_{j}\\
&=\frac{d^{2}}{k^{2}}\left(\frac{1}{\left\Vert W_{i}v_{i}\right\Vert }-1\right)\left(\frac{1}{\left\Vert W_{j}v_{j}\right\Vert }-1\right)v_{i}^{\top}U^{\top}S_{i}^{\top}S_{i}\underbrace{UU^{\top}}_{I}S_{j}^{\top}S_{j}Uv_{j}
\end{align*}

Note that $M:=S_{i}^{\top}S_{i}S_{j}^{\top}S_{j}$ is PSD (both $S_{i}^{\top}S_{i}$
and $S_{j}^{\top}S_{j}$ are diagonal matrices and so is the product). Furthermore, $\E[\frac{d}{k}S_{i}^{\top}S_{i}]=I$. Thus\knote{I got a bit confused here. Are we using $\E[M] \preceq I$ here in the last step?}

\begin{align*}
\E&\left[\frac{d^{2}}{k^{2}}\left(\frac{1}{\left\Vert W_{i}v_{i}\right\Vert }-1\right)\left(\frac{1}{\left\Vert W_{j}v_{j}\right\Vert }-1\right)v_{i}^{\top}U^{\top}MUv_{j}, E_{1}\right]\cdot \Pr\left[E_{1}\right]\\
&\le\E\left[\frac{d^{2}}{4k^{2}}\left(\frac{1}{\left\Vert W_{i}v_{i}\right\Vert }-1\right)\left(\frac{1}{\left\Vert W_{j}v_{j}\right\Vert }-1\right)\left(v_{i}^{\top}+v_{j}^{\top}\right)U^{\top}MU\left(v_{i}+v_{j}\right), E_{1}\right]\cdot \Pr\left[E_{1}\right]\\
&\le\E\left[O\left(\ln^2\left(k/\delta\right)d^{2}/k^{3}\right)\left(v_{i}^{\top}+v_{j}^{\top}\right)U^{\top}MU\left(v_{i}+v_{j}\right) \right]\\
&=O\left(\ln^2\left(k/\delta\right)/k\right)
\end{align*}
Therefore,
\[
\E\left[\left\langle W_{i}^{\top}\left(u_{i}-W_{i}v_{i}\right),W_{j}^{\top}\left(u_{j}-W_{j}v_{j}\right)\right\rangle \right]\le O\left(\ln^2\left(k/\delta\right)/k\right)+O\left(d^{2}/k^{2}\right)\cdot\Pr\left[\overline{E_{1}}\right]
\]  
    We now complete the proof of the claim.  Combining the analysis, we get 
    \begin{align*}
        \E\left[\ltwo{ \hat \mu- \frac{1}{n}\sum_{i=1}^n v_i  }^2 \right] 
        & \le \err_{n,d}(\pug) \cdot \left(1 + O\left( \frac{\diffp + \log k}{k} \right) \right) \\
        & \quad     + O\left(\frac{d}{nk} + \frac{d^2 \delta}{ k^2}+ \frac{\ln^2 (k/\delta)}{k} + \alpha_{\mathcal{W}} \right).
    \end{align*}
    Noticing that $\err_{1,d}(\pug)/n = \err_{n,d}(\pug) = c_{d,\diffp} \cdot \frac{d}{n\diffp}$ for some constant $c_{d,\diffp}$ (see~\cite{AsiFeTa22}), and $\delta \le k/(nd)$, this implies that 
    \begin{align*}
        \E\left[\ltwo{  \hat \mu- \frac{1}{n}\sum_{i=1}^n v_i  }^2 \right] 
        & \le \err_{n,d}(\pug) \cdot \left(1 + O\left( \frac{\diffp + \log k}{k}  + \frac{n \diffp \ln^2 (k/\delta)}{d k} \right)  \right)  + \alpha_{\mathcal{W}}  .
    \end{align*}
    This proves the claim. 
\end{proof}
}

%% file: experiments.tex
\section{Experiments}
\label{sec:experiments}

\iftoggle{arxiv}{
\begin{figure*}[!ht]
  \begin{center}
    \begin{tabular}{cc}
      \begin{overpic}[width=.5\columnwidth]{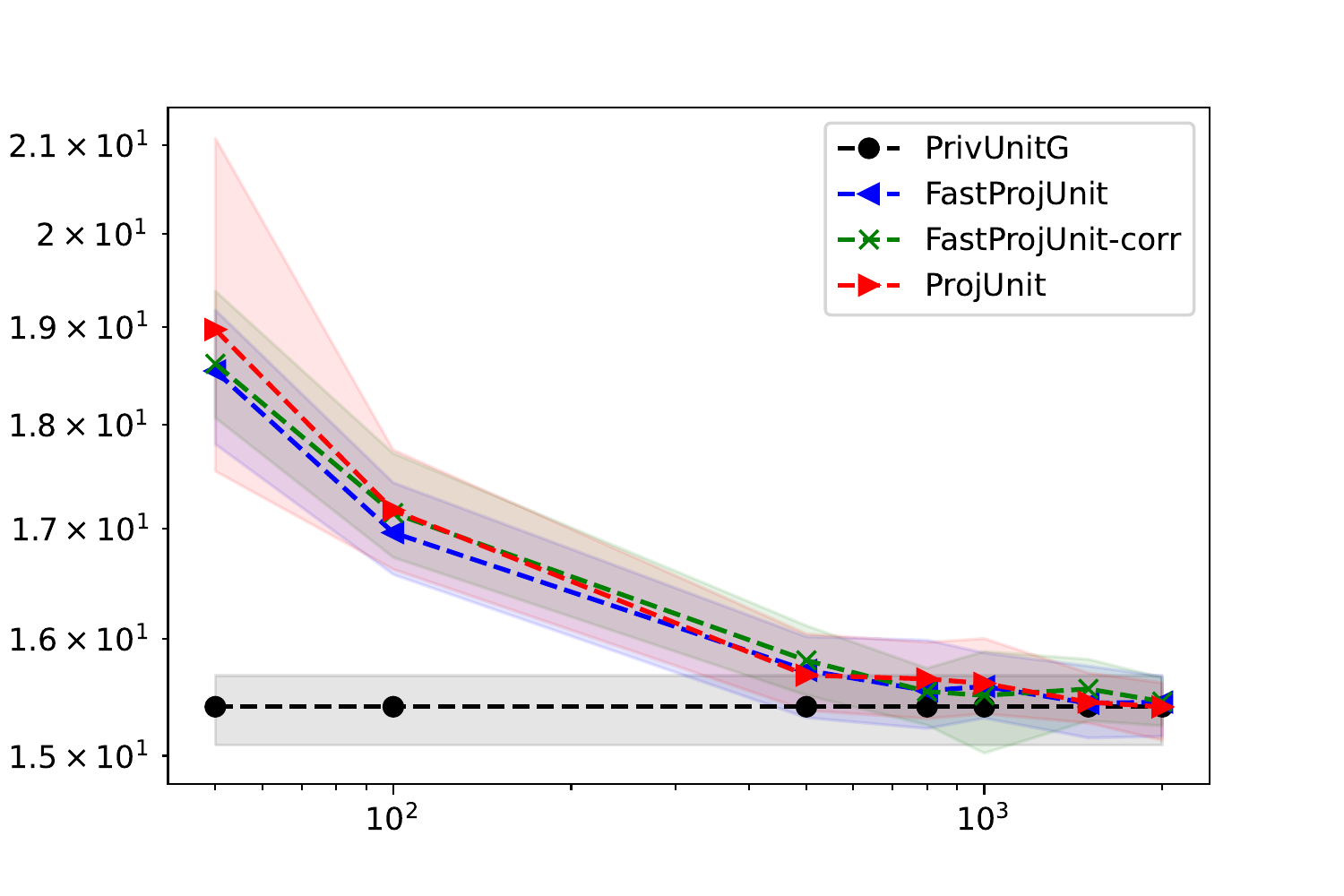}
      	\put(-4.8,20){
			\rotatebox{90} {\small Squared Error}
            \put(48,-18){{\small $k$}}
	}
      \end{overpic} &
      \begin{overpic}[width=.5\columnwidth]{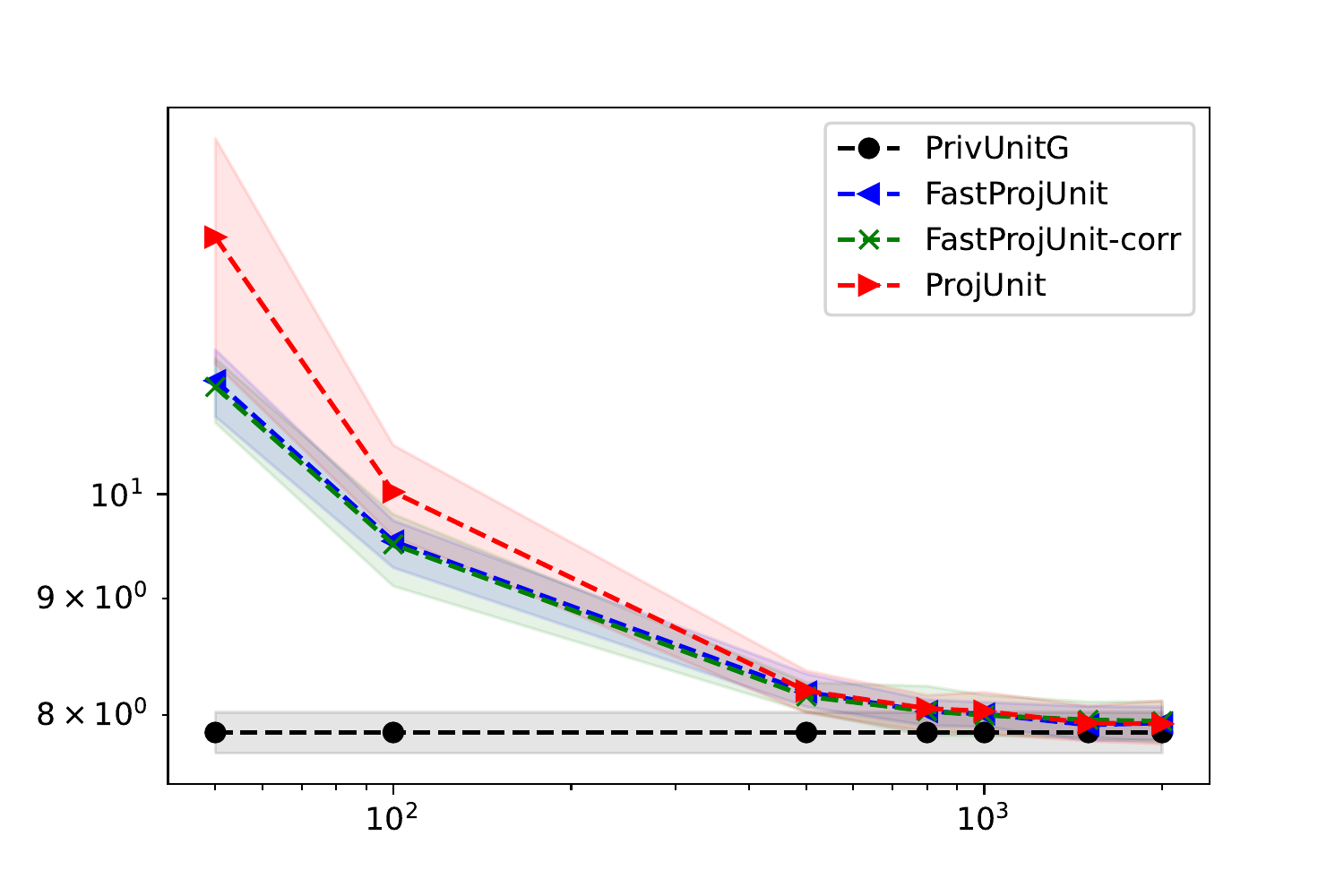}
    	\put(-3.5,20){
			\rotatebox{90} {\small Squared Error}
            \put(48,-18){{\small $k$}}
	}
      \end{overpic}  \\
      (a) & (b)
    \end{tabular}
   \caption{\label{fig:pjl-comm} Performance of \pjl, \pfjl and their correlated versions with 90\% confidence intervals as a function of $k$ for $d = 32768$,  $\nrep = 30$, $n = 50$, and (a) $\diffp = 10$ or (b) $\diffp = 16$. }
   \end{center}
\end{figure*}
}
{

\begin{figure}
\floatbox[{\capbeside\thisfloatsetup{capbesideposition={right,center},capbesidewidth=4cm}}]{figure}
{\caption{\label{fig:pjl-comm} Performance of \pjl, \pfjl and their correlated versions with 90\% confidence intervals as a function of $k$ for $d = 32768$,  $\nrep = 30$, $n = 50$, and  $\diffp = 10$ }}
{\begin{overpic}[width=1.7\columnwidth]{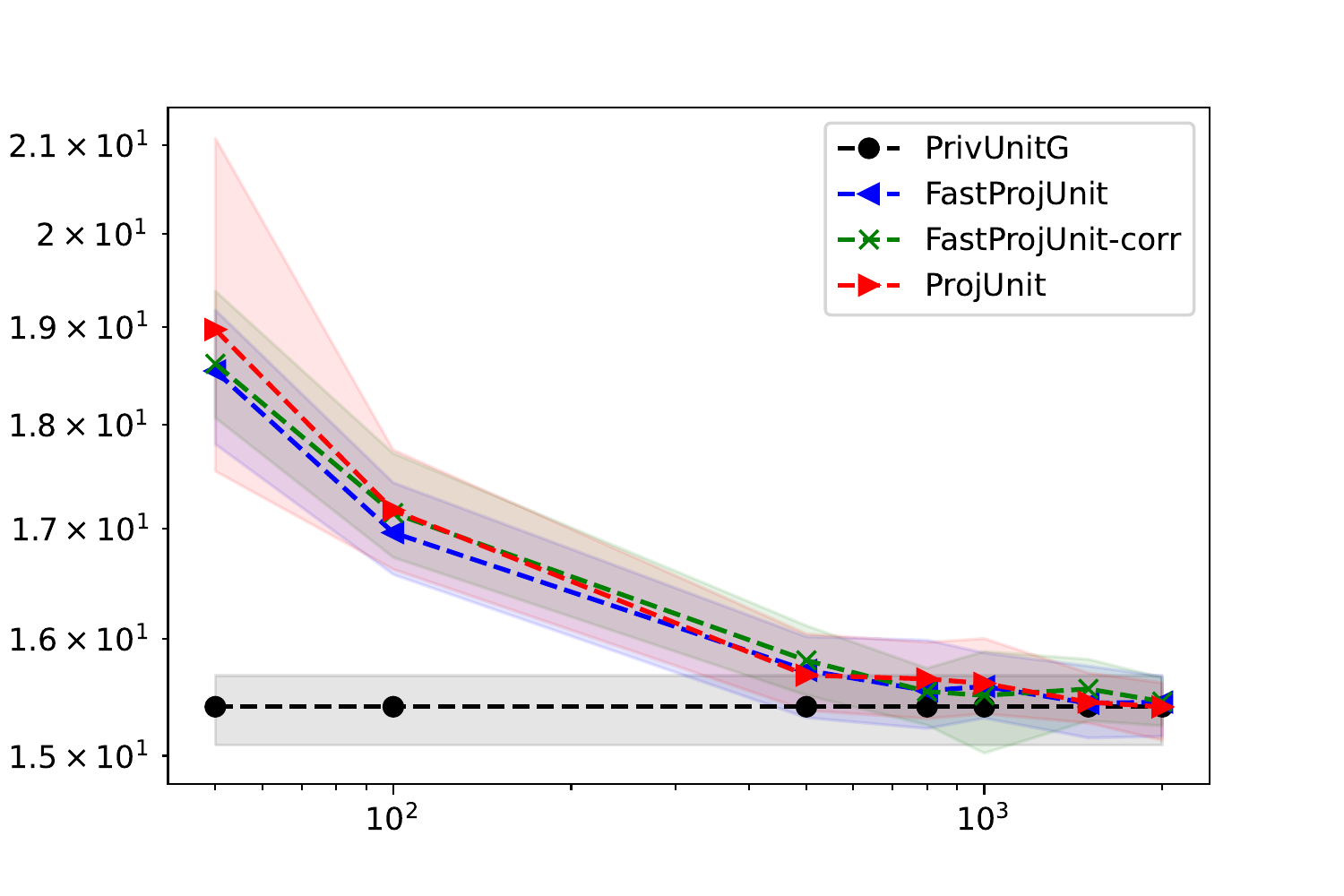}
      	\put(-4.8,20){
			\rotatebox{90} {\small Squared Error}
            \put(48,-18){{\small $k$}}
	}
      \end{overpic} }
\end{figure}

}

We conclude the paper with several experiments that demonstrate the performance of our proposed algorithms, comparing them to existing algorithms in the literature. We conduct our experiments in two different settings: the first is synthetic data, where we aim to test our algorithms and understand their performance for our basic task of private mean estimation, comparing them to other algorithms. In our second setting, we seek to evaluate the performance of our algorithms for \emph{private federated learning} which requires private mean estimation as a subroutine for DP-SGD. 

\subsection{Private mean estimation}
\label{sec:exp-LDP-mean}
In our synthetic-data experiment, we study the basic private mean estimation problem, aspiring to investigate the following aspects:
\begin{enumerate}
    \item Utility of \pjl~algorithms as a function of the communication budget
    \item Utility of our low-communication algorithms compared to the optimal utility and other existing low-communication algorithms
    \item Run-time complexity of our algorithms compared to existing algorithms
\end{enumerate}

Our experiments\footnote{The code is also available online on \url{https://github.com/apple/ml-projunit}} measure the error of different algorithms for estimating the mean of a dataset. To this end, we sample unit vectors $v_1,\dots,v_n \in \R^d$ by normalizing samples from the normal distribution $\normal(\mu,1/d)$ (where $\mu \in \reals^d$ is a random unit vector), and apply a certain privacy protocol $\nrep $ times to estimate the mean $\sum_{i=1}^n v_i /n$, producing mean squared errors $e_1,\dots,e_\nrep$. Our final error estimate is then the mean $\frac{1}{\nrep} \sum_{i=1}^\nrep e_i$. 
We test the performance of several algorithms:

\iftoggle{arxiv}{
\begin{itemize}
    \item \pjl: we use this to denote our basic \pjl~algorithm using random rotations
    \item \pfjl: we use this to denote our fast \pjl~algorithm using the SRHT distribution
    \item \pfjl-corr: we use this to denote our fast \pjl~algorithm with correlated transformations with $G=1$ (\Cref{sec:algs-corr})
    \item \pug: near-optimal algorithm~\cite{AsiFeTa22}
    \item Comp\pug: near-optimal compressed version of \pug~\cite{AsiFeTa22,FeldmanTa21}
    \item \phs: low-communication algorithm~\cite{DuchiJW18}
    \item Re\phs: multiple repetitions of the \phs ~algorithm~\cite{DuchiJW18,FeldmanTa21}
    \item \sqkr: order optimal low-communication algorithm~\cite{ChenKO20}
\end{itemize}
}
{
\pjl (\Cref{sec:algs-jl}), \pfjl (\Cref{sec:algs-fjl}), \pfjl-corr (\Cref{sec:algs-corr}), \pug~\cite{AsiFeTa22}, Comp\pug~\cite{AsiFeTa22,FeldmanTa21}, \phs~\cite{DuchiJW18}, Re\phs~\cite{DuchiJW18,FeldmanTa21}, \sqkr~\cite{ChenKO20}.
}

In~\Cref{fig:pjl-comm}, we plot the error for our \pjl~algorithms as a function of the communication budget $k$. We consider a high-dimensional regime where $d = 2^{15}$ with a small number of users $n = 50$ and a bounded communication budget $k \in [1,2000]$. Our plot shows that our \pjl~algorithms obtain the same near-optimal error as \pug for $k$ as small as $1000$. Moreover, the plots show that the correlated versions of our \pjl algorithms obtain nearly the same error.

To translate this into concrete numbers, the transform $W$ can be communicated using a small seed ($\sim 128$ bits) in practice, or using $k\log d + \log^2 d \sim 20400$ bits or less than 3kB for $d = 10^6$. Sending the $k$-dimensional vector of 32-bit floats would need an additional 4kB. Thus the total communication cost is between 4 and 8kB. This can be further reduced by using a compressed version of \pug in the projected space, which requires the client to send a 128-bit seed. In this setting, the communication cost is a total of 256 bits. Thus in the sequel, we primarily focus on the $k=1000$ version of our algorithms.


\begin{figure*}[!h]
  \begin{center}
    \begin{tabular}{cc}
      \begin{overpic}[width=.5\columnwidth]{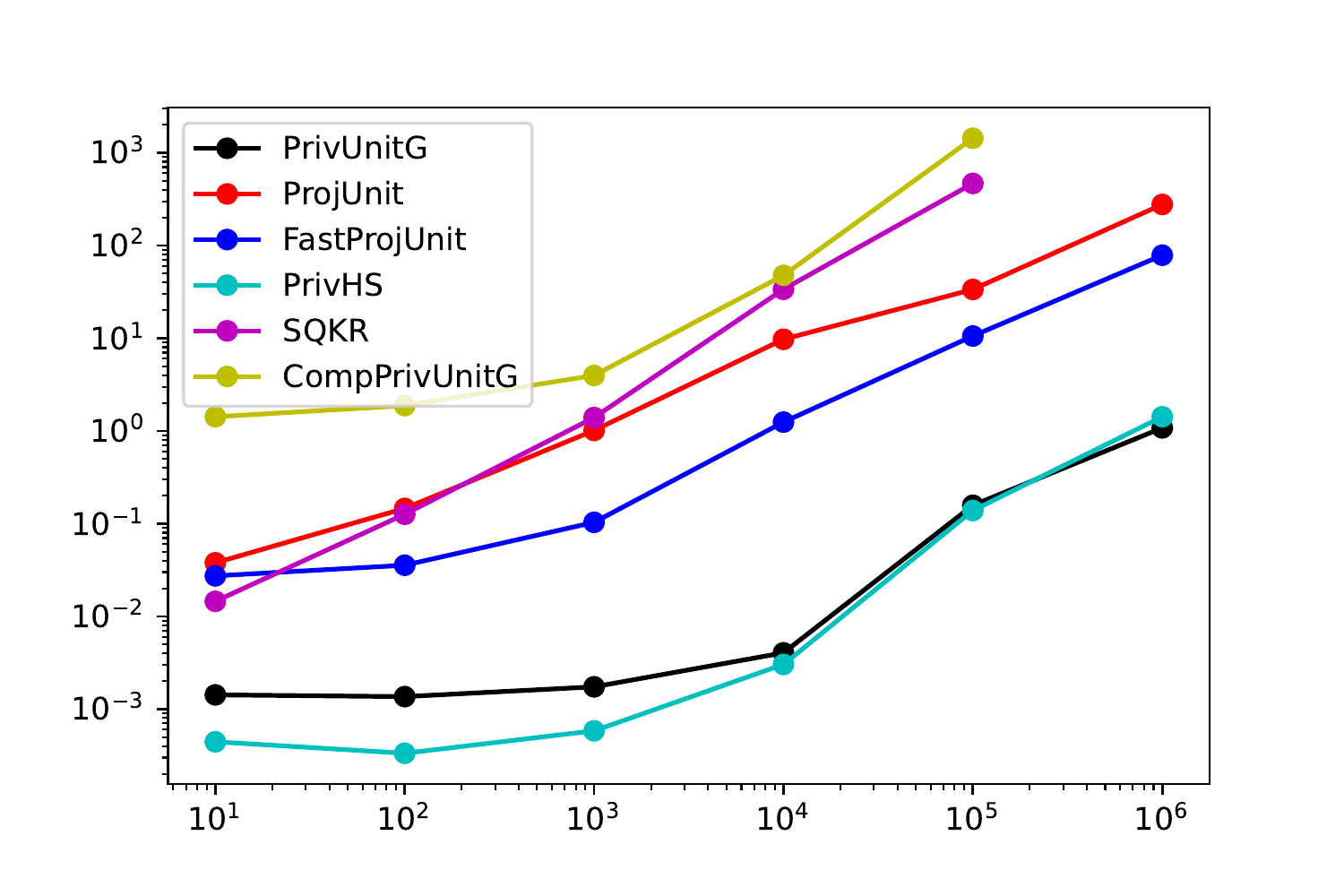}
       \put(-1.5,20){
			\rotatebox{90} {\small Time (seconds)}
	}
         \put(50,0.5){{\small $d$}}
      \end{overpic} &
      \hspace{-1cm}
      \begin{overpic}[width=.5\columnwidth]{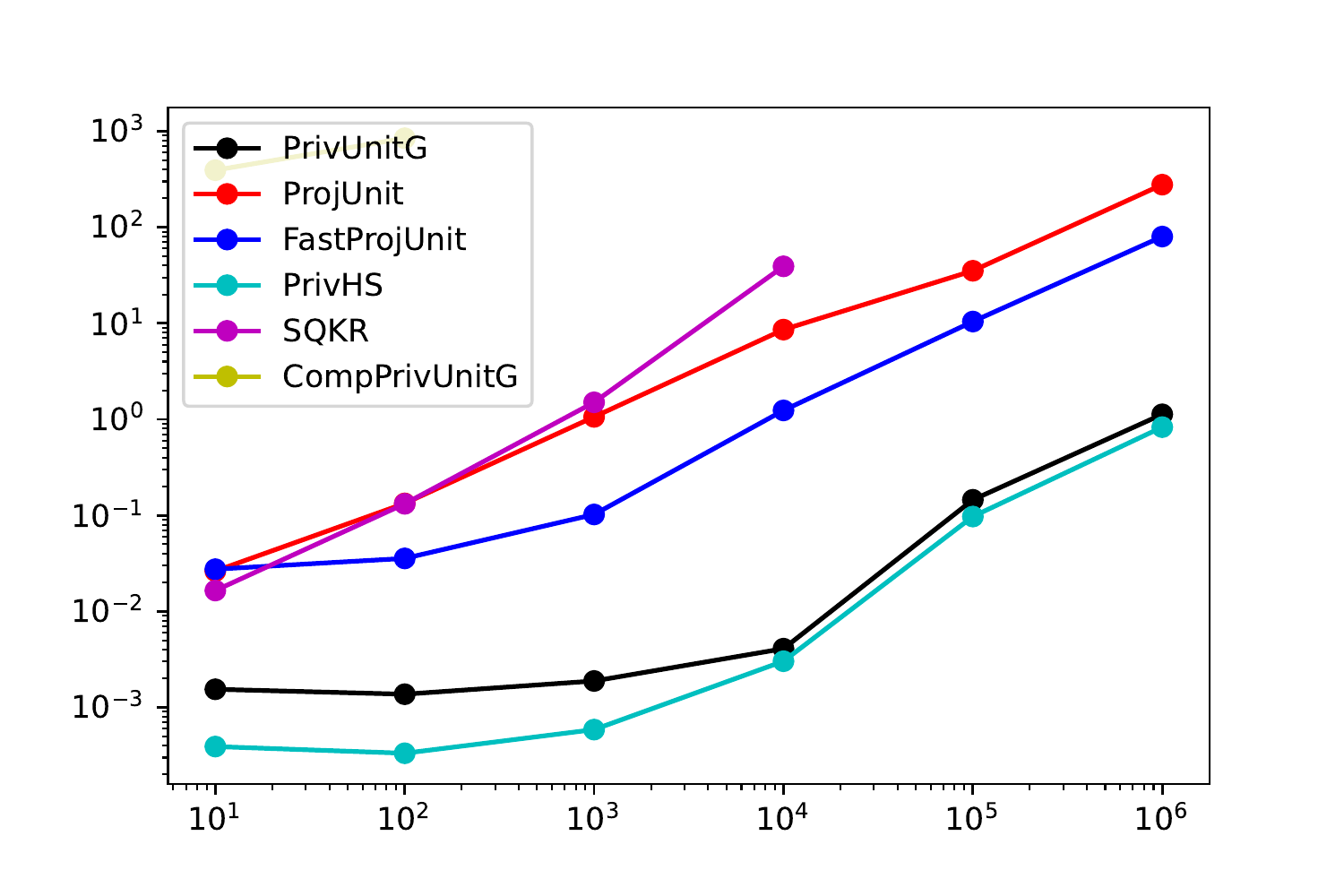}
    \put(-1.5,20){
			\rotatebox{90} {\small Time (seconds)}
	}
        \put(50,0.5){{\small $d$}}
      \end{overpic}  \\
      (a) & (b)
    \end{tabular}
    \vspace{-.2cm}
    \caption{\label{fig:time} Run-time (in seconds) as a function of the dimension for (a) $\diffp = 10$ and (b) $\diffp = 16$. The plots for some algorithms \iftoggle{arxiv}{(e.g. Comp\pug)} are not complete as they did not finish within the cut-off time.}
  \end{center}
\end{figure*}

\iftoggle{arxiv}{
\begin{figure*}[h]
  \begin{center}
    \begin{tabular}{ccc}
      \begin{overpic}[width=.35\columnwidth]{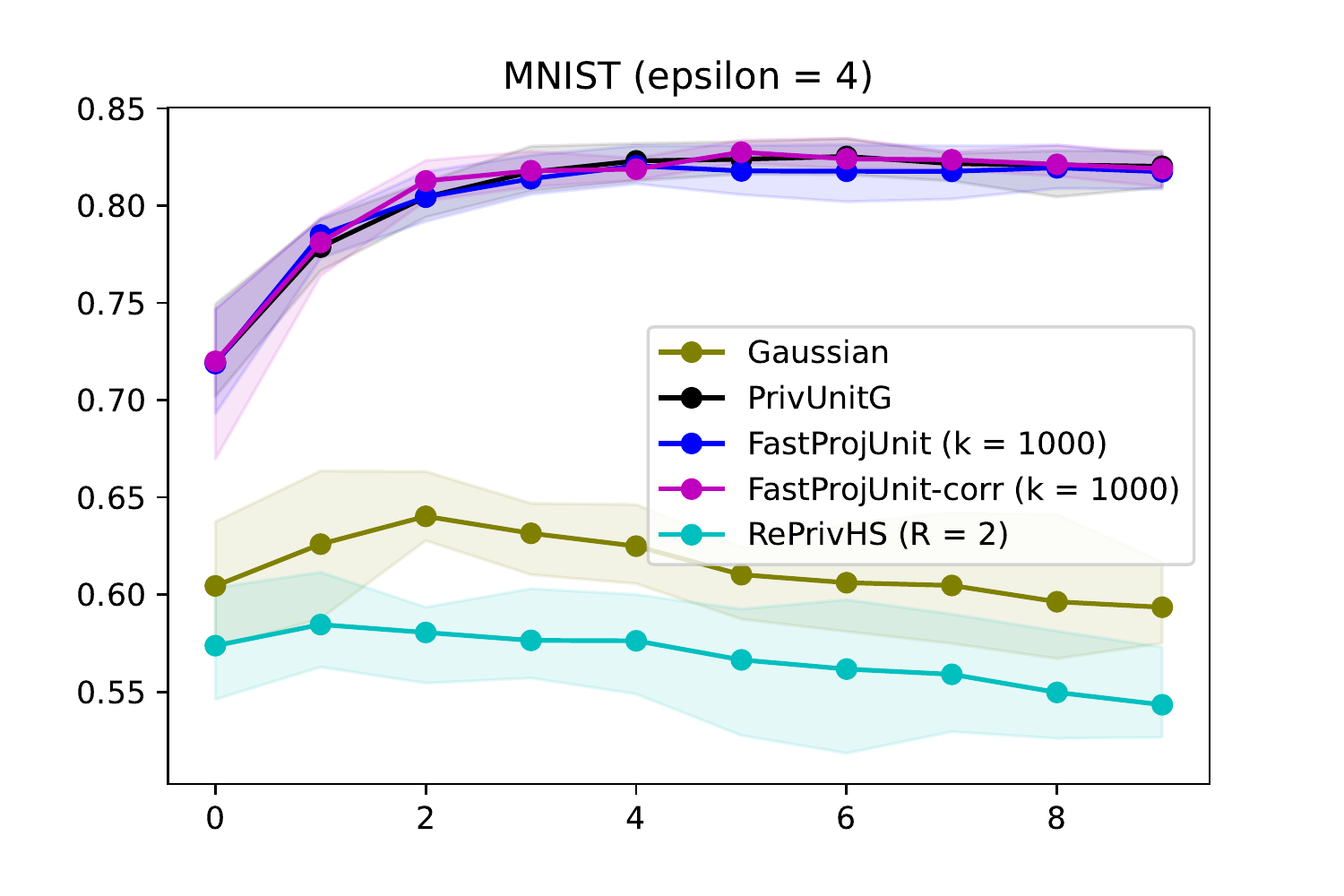}
      \put(-2,18){
			\rotatebox{90} {\small Test accuracy}
	}
         \put(45,-3){{\small Epoch}}
      \end{overpic} &
      \hspace{-1cm}
      \begin{overpic}[width=.35\columnwidth]{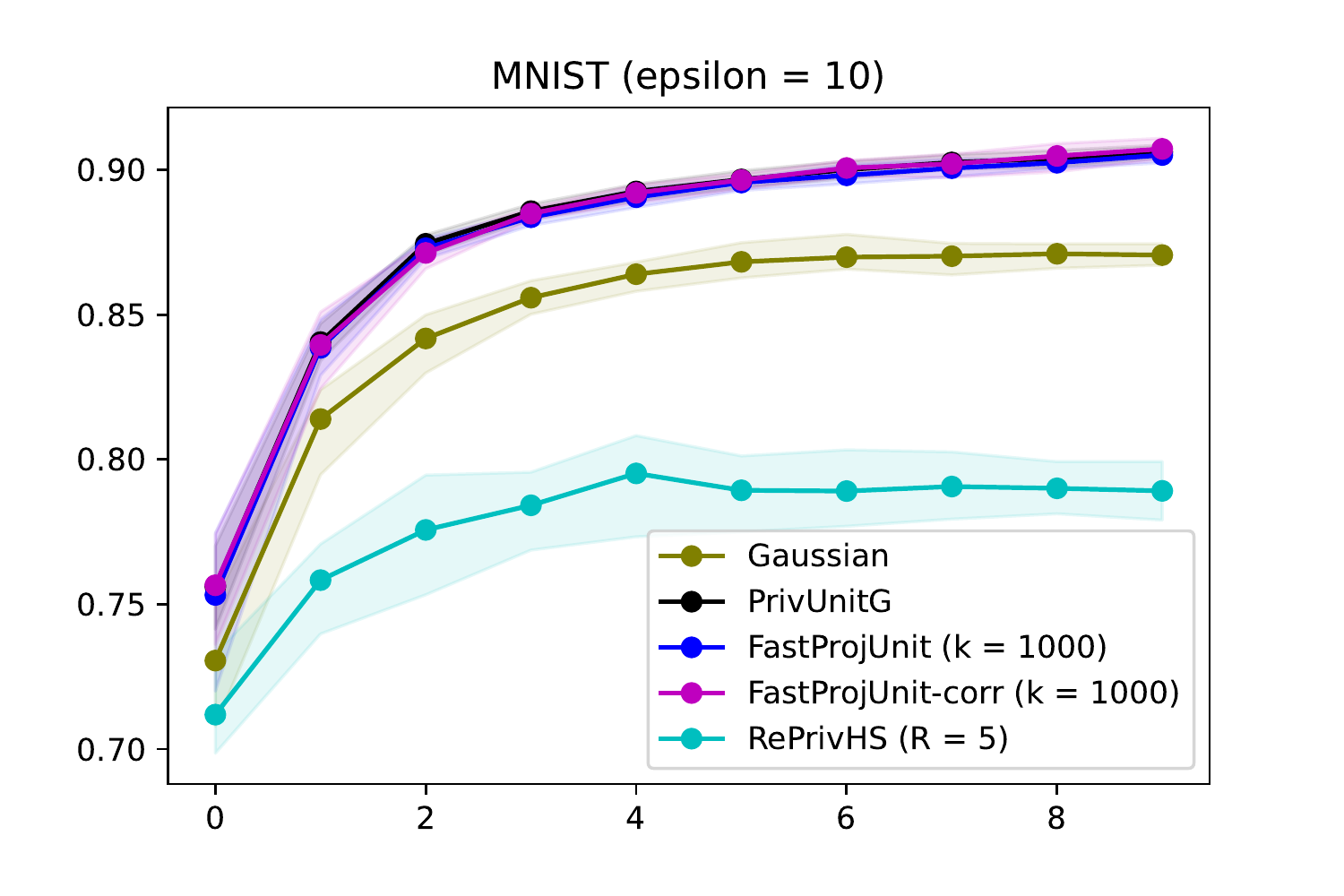}
          \put(-2,18){
			\rotatebox{90} {\small Test accuracy}
	}
         \put(45,-3){{\small Epoch}}
      \end{overpic}  &
      \hspace{-1cm}
      \begin{overpic}[width=.35\columnwidth]{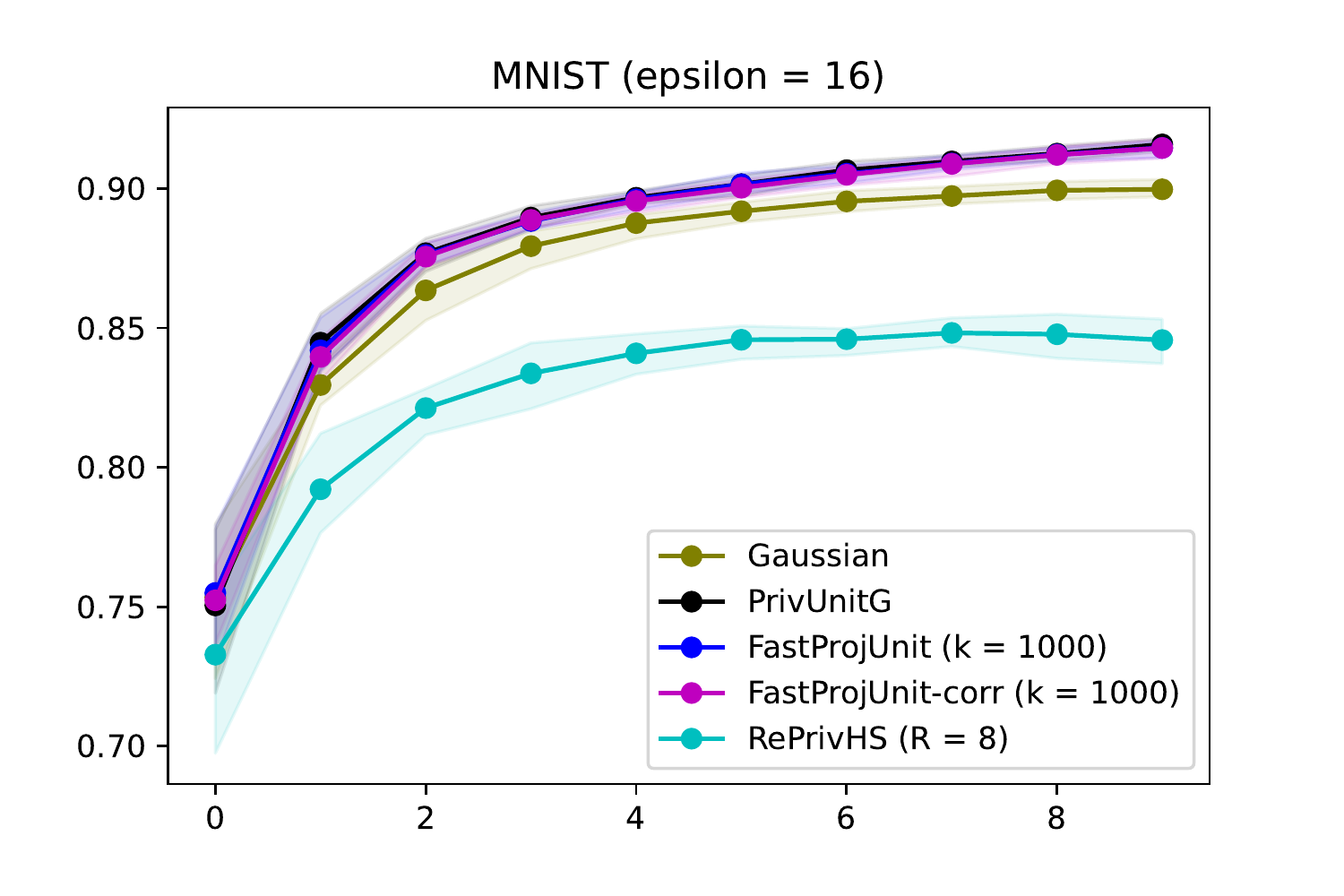}
          \put(-2,18){
			\rotatebox{90} {\small Test accuracy}
	}
         \put(45,-3){{\small Epoch}}
      \end{overpic}  \\
      (a) & (b) & (c)
    \end{tabular}
    \caption{\label{fig:mnist} Test accuracy on the MNIST dataset with 90\% confidence intervals  as a function of epoch for (a) $\diffp = 4.0$, (b) $\diffp = 10.0$ and (c) $\diffp = 16.0$. }
  \end{center}
\end{figure*}
}
{
\begin{figure}
\floatbox[{\capbeside\thisfloatsetup{capbesideposition={right,center},capbesidewidth=4cm}}]{figure}
{\caption{\label{fig:mnist} Test accuracy on the MNIST dataset with 90\% confidence intervals  as a function of epoch for $\diffp = 10.0$.}}
{\begin{overpic}[width=1.8 \columnwidth]{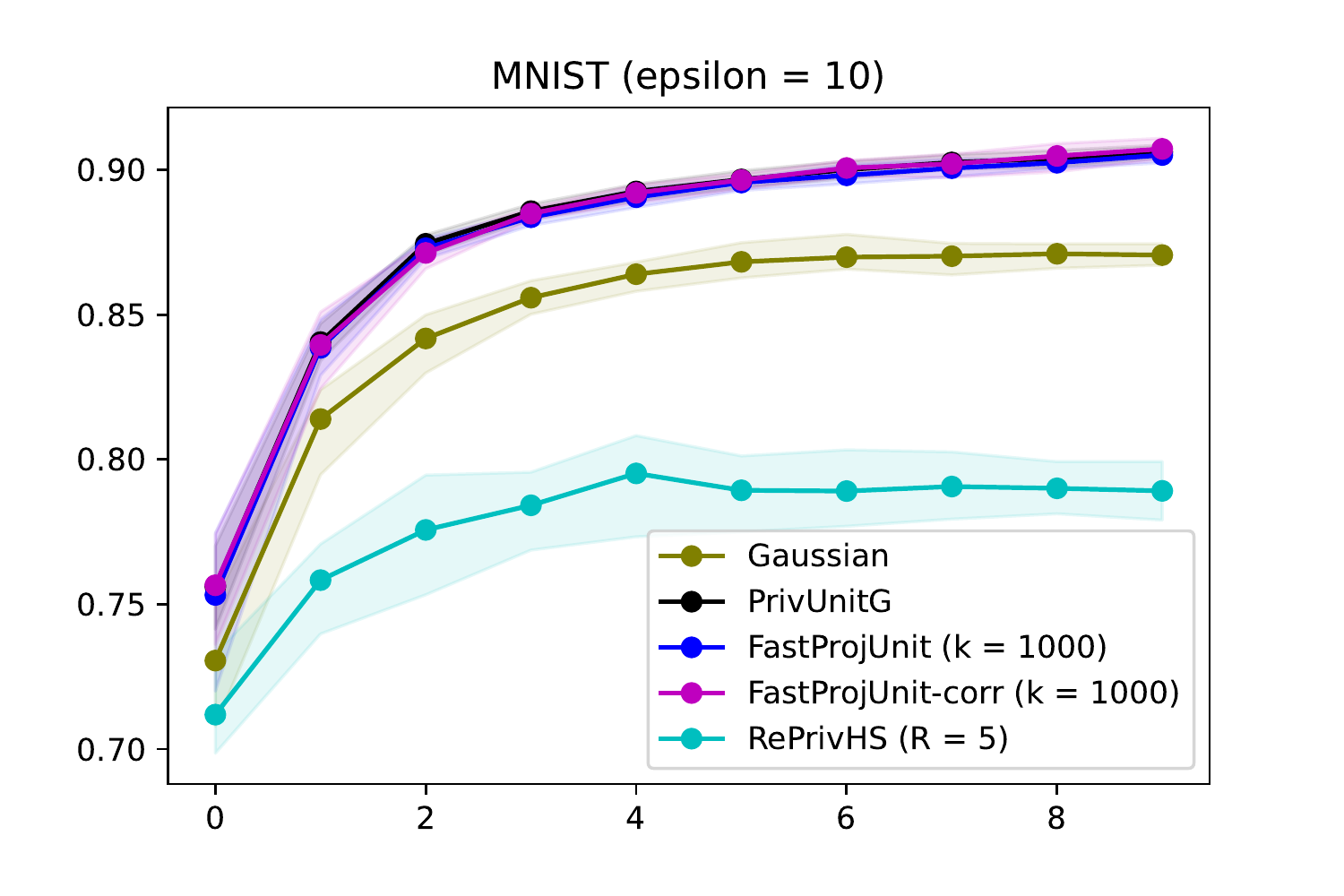}
          \put(-2,18){
			\rotatebox{90} {\small Test accuracy}
	}
         \put(45,-3){{\small Epoch}}
      \end{overpic} }
\end{figure}
}

Next, we compare the performance of our low-communication algorithms against existing low-communication algorithms: \phs~and \sqkr. In~\Cref{fig:err-comp}, we plot the error as a function of the privacy parameter for each algorithm using the best choice of $k$ (bound on communication) for each algorithm. In particular, we choose $k=\diffp$ for \sqkr, num repetitions $R=\diffp/2$ for repeated \phs, $k=1000$ for \pjl~and \pfjl. Moreover, in this experiment we set $n=1$ and $\nrep = 50$ to estimate the variance of each method.
The figure shows that \phs~and \sqkr, while having low-communication complexity, suffer a significantly worse utility than (near-optimal) \pug. On the other side, both our \pjl~algorithms obtain nearly the same error as \pug~with a bounded communication budget of $k = 1000$.

In our third experiment in~\Cref{fig:time}, we plot the runtime of each algorithm as a function of the privacy parameter. 
Here, we use $n=1$, $\nrep = 10$ and measure the run-time of each method for different values of the dimension $d$ and privacy parameter $\diffp$, allowing each method to run for $1$ hour before interrupting.
As expected from our theoretical analysis, the runtime of \pjl~using random rotations is noticeably slower than the (high communication cost) \pug. However, our SRHT-based \pfjl is substantially faster and has a comparable run-time to \pug. 
Moreover, for large $\diffp$ and $d$, the run-time of compressed \pug becomes too costly compared to our algorithms due to the $e^\diffp d$ time complexity.


\subsection{Private federated learning}
\label{sec:PFL}
Having demonstrated the effectiveness of our methods for private mean estimation, in this section we illustrate the improvements offered by our algorithms for private federated learning. Similarly to the experimental setup in~\cite{ChaudhuriGuRa22}, we consider the MNIST~\cite{LeCunCoBu98} dataset and train a convolutional network (see~\Cref{tab:conv}) using DP-SGD~\cite{AbadiChGoMcMiTaZh16} with $10$ epochs and different sub-routines for privately estimating the mean of gradients at each batch. In order to bound the sensitivity, we clip the gradients to have $\ell_2$-norm $1$, and run DP-SGD with batch size of $600$, step-size equal to $0.1$, and momentum of $0.5$\iftoggle{arxiv}{ for all methods.}{.}

\Cref{fig:mnist} shows our results for this experiment, where we plot the test accuracy as a function of the epoch for each method.
The plots demonstrate that our \pjl~algorithms obtain similar performance to \pug, and better performance than the Gaussian mechanism or \phs. For the Gaussian mechanism, we set $\delta = 10^{-5}$ and add noise to satisfy \ed-DP using the analysis in~\cite{BalleWa18}. We did not run \sqkr~in this experiment as it is not sufficiently computationally efficient for this experiment and has substantially worse performance in the experiments of the previous section. 
We also did not run the MVU mechanism~\cite{ChaudhuriGuRa22} as their experiments show that it is worse than the Gaussian mechanism which has worse performance than our methods.

Finally, our private algorithms obtain accuracy roughly 91\%, whereas the same model trained without privacy obtains around 98\%. This degradation in accuracy is mostly due to the choice of the optimization algorithm (DP-SGD with clipping); indeed, even without adding any noise, DP-SGD with clipping achieves around 91\% accuracy, suggesting that other private optimization algorithms with different clipping strategies (e.g.~\cite{PichapatiSYRK20,AsiDuFaJaTa21}) may tighten this gap further. As this is not the main focus of our work, we leave this investigation to future work.


\iftoggle{arxiv}{
\renewcommand{\arraystretch}{1}
\begin{table}[h]
\begin{center}
\begin{tabular}{ll}
            \hline
            \textbf{Layer} & \textbf{Parameters} \\
            \hline
            Convolution $+\tanh$ & 16 filters of $8 \times 8$, stride 2, padding 2 \\
            Average pooling & $2 \times 2$, stride 1 \\
            Convolution $+\tanh$ & 32 filters of $4 \times 4$, stride 2, padding 0 \\
            Average pooling & $2 \times 2$, stride 1 \\
            Fully connected $+\tanh$ & 32 units \\
            Fully connected $+\tanh$ & 10 units \\
            \hline
\end{tabular}
\end{center}
\caption{Architecture for convolutional network model.}
\label{tab:conv}
\end{table}
\renewcommand{\arraystretch}{3}
}






%% file: privUnitG.tex
\section{Details of \pug}
For completeness, in this section we provide the full details of \pug which was proposed by~\citet{AsiFeTa22}. Roughly, this algorithm uses the normal distribution to approximate the uniform distribution over the sphere for large dimensions. We refer the reader to~\cite{AsiFeTa22} for more details about \pug.

In the algorithm, $\Phi$ and $\phi$ denote the Cumulative distribution function and probability density function for a Gaussian random variable $\normal(0,I_d)$.
There are multiple ways to set the parameters of $\pug$ to achieve $\diffp$-DP; in our paper, we use the optimized parameters as described by~\citet{AsiFeTa22}, which allow to minimize the expected mean squared error (see Proposition 4 in~\cite{AsiFeTa22}).

We note that~\Cref{alg:pug} describes the clients' algorithm (local randomizers) in the \pug protocol. The server aggregation simply adds all messages received from clients. Thus, we let $\R_{\pug_\diffp}$ denote the local randomizer in~\Cref{alg:pug} (with optimized parameters to satisfy $\diffp$-DP) and let $\mathcal A_{\pug_\diffp}$ denote the additive server aggregation. 
\begin{algorithm}
	\caption{$\pug(p,q)$}
	\label{alg:pug}
	\begin{algorithmic}[1]
		\REQUIRE $v \in \sphere^{d-1}$, $q \in [0,1]$, $p \in [0,1]$
		\STATE Draw $z \sim \mathsf{Ber}(p)$
		\STATE Let $U = \normal(0,\sigma^2) $ where $\sigma^2 = 1/d$
		\STATE Set $\gamma = \Phi_{\sigma^2}^{-1}(q) = \sigma \cdot \Phi^{-1}(q)$
		\IF{$z=1$}
		    \STATE Draw $\alpha \sim U \mid U \ge \gamma$
		\ELSE
		     \STATE Draw $\alpha \sim U \mid U < \gamma$
		\ENDIF
		\STATE Draw $V^{\perp} \sim \normal(0, \sigma^2(I - v v^T)$
		\STATE Set $V = \alpha v + V^{\perp}$
		\STATE Calculate
		\begin{equation*}
		    m = \sigma \phi(\gamma/\sigma) \left(\frac{p}{1 - q}  - \frac{1-p}{q} \right)
		\end{equation*}
        \STATE Return $\frac{1}{m} \cdot V$
	\end{algorithmic}
\end{algorithm}

We also use the following useful result on the error of \pug for different dimensions. Recall that $ \err_{n,d}(\pug)  = c_{d,\diffp} \frac{d}{n \diffp}$. Then we have the following.
\begin{proposition}[Propsition 5, \cite{AsiFeTa22}]
\label{prop:C_eps}
    Fix $\diffp > 0$.
    For any $1 \le k \le d$,
    \begin{equation*}
        1 - O \left(\frac{\diffp + \log k}{k} + \frac{\diffp}{k}\right) \le
        \frac{c_{k,\diffp}}{c_{d,\diffp}}
        \le 1 + O \left( \frac{\diffp + \log k}{k} + \frac{\diffp}{k}  \right).
    \end{equation*}
\end{proposition}

%% file: algorithms-unbiased.tex
\section{Nearly unbiased \pjl~randomizers}
\label{sec:algs-unb}
Our randomizers in~\Cref{sec:algs} and~\Cref{sec:algs-corr} may have $O(\log(d)/k)$ bias which can become relatively large when $n \gg d$. 
In this section, we propose a different normalization technique which allows to provide a sufficiently small bound on the bias, while still enjoying the same guarantees as the fast \pjl algorithm. We develop versions of this algorithm for both random rotations (\Cref{sec:gauss-unb}) and the SRHT transform (\Cref{sec:had-unb}).

\subsection{Unbiased variant of \pjl~using random rotations}
\label{sec:gauss-unb}

In this section, we describe the modification for random rotation matrices. These transformations are not as efficient as SRHT hence we only present the simple non-correlated version; in the next section we present our unbiased and correlated sampling procedure for SRHT transforms.

For rotationally symmetric distributions of matrices, we slightly modify the algorithm by scaling the output of $\pug$ by a fixed factor $c$ so that it leads to an unbiased estimate of $v$ i.e. $\E[\ts^\top \hat{u}]=v$.

\begin{algorithm}
	\caption{Unbiased version of \pjl using random rotations (client)}
	\label{alg:privJL-cl-rot-unb}
	\begin{algorithmic}[1]
		\REQUIRE Input vector $v \in \reals^d$. 
            \STATE Randomly sample a rotation matrix $\ts \in \reals^{k \times d}$ as described in~\eqref{eq:ts-rot}
            \STATE Project the input vector $v_p = \ts v$
            \STATE Normalize $u = \frac{v_p}{\ltwo{v_p}}$ 
            \STATE Let $\hat u = {\color{red} c} \cdot \pug(u)$ where $c = \sqrt{\frac{k}{d}}\frac{\Gamma\left(\left(d+1\right)/2\right)\Gamma\left(k/2\right)}            {\Gamma\left(\left(k+1\right)/2\right)\Gamma\left(d/2\right)}$ 
            \STATE Send $\hat u$ and (encoding of) $\ts$ to server
	\end{algorithmic}
\end{algorithm}

We present the details of this modification in~\Cref{alg:privJL-cl-rot-unb} and state its guarantees in the  following theorem. We let $\Rjl$ denote the local randomizer described in~\Cref{alg:privJL-cl-rot-unb}.
\begin{theorem}
\label{thm:err-unbiased-gaussian}
    Let $k \le d$. For all unit vectors $v_1,\ldots,v_n \in \reals^d$, setting $\hat \mu = \Ajl \left( \Rjl(v_1), \dots, \Rjl(v_n) \right)$, the local randomizers $\Rjl$ are $\diffp$-DP and
    \begin{align*}
        \E\left[\ltwo{ \hat \mu - \frac{1}{n}\sum_{i=1}^n v_i  }^2 \right] 
        \le \err_{n,d}(\pug) \cdot \left(1 + O\left(\frac{\diffp + \log k}{k} \right) \right).
    \end{align*}
\end{theorem}
\begin{proof}
    The proof proceeds in the same way as Theorem~\ref{thm:err-jl-gen}. We break the error down into two terms:
    \begin{align*}
    \E\left[\ltwo{\hat \mu - \frac{1}{n}\sum_{i=1}^n v_i}^2\right] 
        & =  \E\left[\ltwo{\frac{1}{n}\sum_{i=1}^n \ts_i^\top \hat{u}_i - v_i}^2\right] \\
        & =  \E\left[\ltwo{\frac{1}{n}\sum_{i=1}^n \ts_i^\top \hat{u}_i - c\ts_i^\top u_i + c\ts_i^\top u_i - v_i}^2\right] \\
        & \stackrel{(i)}{=}  \frac{1}{n^2}\sum_{i=1}^n \E\left[\ltwo{\ts_i^\top \hat{u}_i - c\ts_i^\top u_i}^2\right] + \frac{1}{n^2} \E\left[\ltwo{\sum_{i=1}^n c\ts_i^\top u_i - v_i}^2\right] \\
        & \le \frac{c^2}{n} \max_{i \in [n]}\E\left[\ltwo{\ts_i^\top}^2\right] \cdot \err_{1,k}(\pug) + \frac{1}{n^2} \E\left[ \ltwo{\sum_{i=1}^n c\ts_i^\top u_i - v_i}^2\right].
    \end{align*}
    where $(i)$ follows from the fact that $\pug$ is unbiased and $\E[\hat u_i] = cu_i$.
    The first term is bounded in the same way as before noting that $c = 1+O(1/k)$ (see~\Cref{lemma:unit-proj}). To analyze the second term, we first show that $\E[c\ts_i^\top u_i] = v_i$ using a change of variables. Let $\ts'_i = \ts_i P_i^\top$ where $P_i$ is the rotation matrix such that $P_i v_i = e_1$, the first standard basis vector. Due to the rotational symmetry of the uniform distribution over rotation matrices, $\ts'_i$ is also a random rotation matrix. Note that $\ts_i = \ts'_i P_i$, hence

    \begin{align*}
    \E_{\ts_i}[c\ts_i^\top u_i] &= \E_{{\ts'_i}}\left[\frac{c}{\ltwo{\ts'_i P_i v_i}}P_i^\top{\ts'}_i^\top \ts'_i P_i v_i\right]\\
        &=c P_i^\top\E\left[\underbrace{\frac{1}{\ltwo{\ts'_i e_1}} {\ts'}_i^\top \ts'_i e_1}_z\right]
    \end{align*}
    Notice that $z_j = \frac{1}{\ltwo{\ts'_i e_1}} e_j^\top {\ts'}_i^\top \ts'_i e_1 = \langle \ts'_i e_j, \frac{1}{\ltwo{\ts'_i e_1}} \ts'_i e_1\rangle$. Because $\ts'_i$ is a random rotation matrix (re-scaled by $\sqrt{d/k}$), 
    \Cref{lemma:unit-proj} implies that $\E[z_1] = \sqrt{\frac{d}{k}}\frac{\Gamma\left(\left(k+1\right)/2\right)\Gamma\left(d/2\right)}{\Gamma\left(\left(d+1\right)/2\right)\Gamma\left(k/2\right)} = \frac{1}{c} = 1 + O(1/k)$ and $\E[z_j]=0$ for all $j>1$. Thus, $\E[z] = \frac{1}{c}e_1$ and $\E[c\ts_i^\top u_i] = P_i^\top e_1 = P_i^\top P_i v_i = v_i$.
    Because $c\ts_i^\top u_i$ is an unbiased estimator of $v_i$, we have
    \begin{align*}
    \E\left[ \ltwo{\sum_{i=1}^n c\ts_i^\top u_i - v_i}^2\right] &= \sum_{i=1}^n \E\left[ \ltwo{c\ts_i^\top u_i - v_i}^2\right]\\
    &=\sum_{i=1}^n\E\left[ \ltwo{c\ts_i^\top u_i}^2 +\ltwo{v_i}^2 - 2 c v_i^\top \ts_i^\top u_i\right]\\
    &=\sum_{i=1}^n\E\left[ \ltwo{c\ts_i^\top u_i}^2 +\ltwo{v_i}^2 - 2 c\ltwo{\ts_i^\top v_i}\right]\\
    &\le n + \sum_{i=1}^n c^2\E\left[ \ltwo{\ts_i^\top}^2\right] \\
    & \le O ( nd/k) .
    \end{align*}
    Combining all of these together, the claim follows by noting that $\err_{1,d}(\pug)/n = \err_{n,d}(\pug) = \Theta(d/n
    \diffp)$.
\end{proof}

\subsection{Nearly unbiased SRHT-based randomizers}
\label{sec:had-unb}
While rescaling by a constant was sufficient to debias the random rotation based randomizer, it is not clear whether such rescaling can debias the SRHT \pjl randomizer as it is not rotationally symmetric. To address this, we propose a different normalization technique for the SRHT randomizer which allows to provide tighter upper bounds on the bias. We provide the details for our new client and server protocols in~\Cref{alg:privJL-cl-unb} and~\Cref{alg:privJL-serv-unb}, respectively.

\begin{algorithm}
	\caption{Nearly Unbiased \pjl (client)}
	\label{alg:privJL-cl-unb}
	\begin{algorithmic}[1]
		\REQUIRE Input vector $v \in \reals^d$, Bias bound probability $\delta$.
            \STATE Randomly sample diagonal $D$ from the Rademacher distribution based on predefined seed 
            \STATE Sample $S \in \reals^{k \times d}$ where each row is chosen uniformly at random without replacement from standard basis vectors $\{e_1,\dots,e_d \}$
            \STATE Let $\ts = \sqrt{d/k} S H D$
            \STATE Set $C = 1 + \Theta(\sqrt{\log^2(k/\delta)/k})$
            \STATE Project the input vector $v_p = \ts v$
            \STATE Complete the norm then normalize:     
            \begin{equation*}
                u = 
                \begin{cases}
                 \frac{1}{\sqrt{C}} \left(v_p, \sqrt{ C- \ltwo{v_p}^2} \right),   & \text{if } \ltwo{v_p}^2 \le C \\
                 \left(\frac{v_p}{\ltwo{v_p}},0 \right), & \text{otherwise}
                \end{cases}
            \end{equation*}
            \STATE Let $\hat u = \pug(u)$
            \STATE Send $C, \hat u$ and (encoding of) $S$ to server
    	\end{algorithmic}
\end{algorithm}

\begin{algorithm}
	\caption{Nearly Unbiased \pjl (server)}
	\label{alg:privJL-serv-unb}
	\begin{algorithmic}[1]
            \STATE Receive $C$, $\hat u_1, \dots, \hat u_1$, from clients with (encodings of) transforms $S_1,\dots,S_n$
            \STATE Sample the diagonal matrices $D$ from the Rademacher distribution based on predefined seed 
            \STATE Let $U = H D$ for where $H \in \reals^{d \times d}$ is the Hadamard matrix
            \STATE Return the estimate 
            \begin{equation*}
                \hat \mu = \frac{\sqrt{C}}{n}  U^\top
            \sum_{i =1}^n S_i^\top \hat u_i[1:k]
            \end{equation*}
	\end{algorithmic}
\end{algorithm}



Let $\Rjlu$ denote the unbiased \pjl local randomizer of the client (\Cref{alg:privJL-cl-unb}), and $\Ajlu$ denote the server aggregation of unbiased \pjl (\Cref{alg:privJL-serv-unb}).
We have the following guarantees for this procedure.

\begin{theorem}
\label{thm:err-unb-fjl}
    Let $k \le d$ and $\delta = k/n^2d$. Assume $k \ge \max\{\diffp + \log k, \log^2(nd)\}$.
    Then for all unit vectors $v_1,\ldots,v_n \in \reals^d$, setting $\hat \mu =   \Ajlu \left( \Rjlu(v_1), \dots, \Rjlu(v_n) \right)$, the local randomizers $\Rjlu$ are $\diffp$-DP and
    \begin{equation*}
    \E\left[\ltwo{ \hat \mu -  \frac{1}{n}\sum_{i=1}^n  v_i  }^2 \right] \le \err_{n,d}(\pug) \cdot \left(1 + O\left( \frac{\diffp + \log k}{k} + \sqrt{\frac{\log^2(nd)}{k}} \right) \right) .
    \end{equation*}
\end{theorem}

\begin{proof}

Note that $ \hat \mu = \frac{\sqrt{C}}{n} \sum_{i=1}^n \ts_i^\top \hat u_i[1:k]$. Thus we get 
    \begin{align}
    & \E \left[\ltwo{\hat \mu - \frac{1}{n} \sum_{i=1}^n v_i}^2 \right] \\
        & = \E\left[\ltwo{\frac{\sqrt{C}}{n} \sum_{i=1}^n \ts_i^\top \hat u_i[1:k]- v_i}^2 \right] \nonumber \\
        & = \frac{1}{n} \max_{i \in [n]} \E\left[\ltwo{ \sqrt{C}\ts_i^\top \hat u_i[1:k]- v_i}^2 \right] + \frac{1}{n^2} \sum_{ i \ne j}  \E\<\sqrt{C}\ts_i^\top \hat u_i[1:k]- v_i, \sqrt{C}\ts_j^\top \hat u_j[1:k]- v_j\>  \label{eq:main-err-unb} 
    \end{align}
    Now we upper bound both terms in~\eqref{eq:main-err-unb} separately. Note that~\Cref{cor:cnw} implies that with probability at least $1 - n\delta$ we have $\ltwo{\ts_i v_i}^2 \le (1+C_1\sqrt{\log^2(k/\delta)/k}) \ltwo{v_i}$ for all $i \in [n]$. We let $E$ denote the event that this event holds. Note that $P(E) \ge 1 - \bar \delta$ where $\bar \delta = n\delta$.

    We begin with the second term in~\eqref{eq:main-err-unb}. Let $C = 1 + C_1\sqrt{\log^2(k/\delta)/k}$ for appropriate constant $C_1>0$. Note that $\ts_i = \sqrt{d/k} S_i H D$, taking expecations over the randomness of the local randomizer, and noticing that \pug is unbiased, we have that for $i \ne j$
    \begin{align}
      \E\<\sqrt{C}\ts_i^\top \hat u_i[1:k]- v_i, \sqrt{C}\ts_j^\top \hat u_j[1:k]- v_j\>  = \E \<\sqrt{C}\ts_i^\top  u_i[1:k]- v_i, \sqrt{C}\ts_j^\top u_j[1:k]- v_j\>  
    \end{align}
    Since $\ts_i = \sqrt{d/k} S_i H D$ and $\ts_j = \sqrt{d/k} S_j H D$, we have
    \begin{align*}
    & \E\left[ \<\sqrt{C}\ts_i^\top  u_i[1:k]- v_i, \sqrt{C}\ts_j^\top  u_j[1:k]- v_j\>  \right] \\
        & = \E\left[ \<\ts_i^\top  \ts_i v_i - v_i, \ts_j^\top \ts_j v_j- v_j\> \mid E  \right] P(E) \\
        & + \E\left[ \<\sqrt{C} \ts_i^\top  \ts_i v_i / \ltwo{\ts_i v_i}- v_i, \sqrt{C}\ts_j^\top  \ts_j v_j /\ltwo{\ts_j v_j}- v_j\> \mid E^c \right] P(E^c) \\
        & \le v_i^T \E[ (\ts_i^\top \ts_i - I) (\ts_j^\top \ts_j - I) \mid E] v_j 
         + \bar \delta (2 Cd/k + 2) \\
        & \le v_i^T \E[ D H (d/k S_i^\top S_i  - I) (d/k S_j^\top S_j  - I) H D \mid E] v_j 
         + \bar \delta (2 Cd/k + 2) \\
         & \stackrel{(i)}{\le} \ltwo{\E[d/k S_i^\top S_i - I \mid E ]} \ltwo{\E[d/k S_j^\top S_j - I \mid E ]} + \bar \delta (2 Cd/k + 2)  \\
         & \le O( (\bar \delta d/k)^2 )  + \bar \delta (2 Cd/k + 2) 
         \le O(\bar \delta d/k),
    \end{align*}
    where inequality $(i)$ follows since $\ltwo{\E[d/k S_j^\top S_j - I \mid E ]} \le O(\bar \delta/k)$ since
    \begin{equation*}
    I 
        = \E[(d/k) S_i^\top S_i ] 
        = \E[(d/k) S_i^\top S_i \mid E] \P(E) + (1-\P(E))\E[(d/k) S_i^\top S_i \mid E^c]
    \end{equation*}
    which implies that 
    \begin{equation*}
    \E[(d/k) S_i^\top S_i \mid E] 
        = \frac{I - (1-\P(E))\E[(d/k) S_i^\top S_i \mid E^c]}{\P(E)}.
    \end{equation*}
    Since $\P(E) \ge 1-\bar \delta$ and $\ltwo{(d/k) S_i^\top S_i} \le (d/k)$, this shows that $\ltwo{\left({\E\left[(d/k) S_i^\top S_i \mid E \right]}  - I\right)  v} = O(\bar \delta d/k) $.

    Now we proceed to analyze the first term in~\eqref{eq:main-err-unb}. 
    Note that for any $i \in [n]$
    \begin{align*}
        \E\left[\ltwo{\sqrt{C} \ts_i^\top \hat u_i[1:k] - v_i}^2 \right] 
        & =  \E\left[\ltwo{\sqrt{C} \ts_i^\top \hat u_i[1:k] - \sqrt{C}\ts_i^\top u_i[1:k] + \sqrt{C}\ts_i^\top u_i[1:k] - v_i}^2 \right] \\
        & \stackrel{(i)}{=}   C \E\left[\ltwo{\ts_i^\top \hat u_i[1:k] - \ts_i^\top u_i[1:k]}^2 \right] + \E\left[ \ltwo{\sqrt{C}\ts_i^\top u_i[1:k] - v_i}^2 \right] \\
        &  \stackrel{(ii)}{\le} C \E\left[\ltwo{\ts_i^\top}^2 \right] \cdot C \cdot \err_{1,k+1}(\pug) + \E\left[ \ltwo{\sqrt{C} \ts_i^\top u_i[1:k] - v_i}^2 \right] .
    \end{align*}
    where $(i)$ follows since $\E[\hat u_i] = u_i$ as $\pug$ is unbiased, and $(ii)$ since \pug is applied for $k+1$ dimensional vectors of squared norm $C$, hence its error is $C \cdot \err_{1,k+1}(\pug)$. For the first term, as $\ltwo{\ts_i}^2 \le d/k$ and $C = 1 + C_1\sqrt{\log^2(k/\delta)/k}$, we have:
    \begin{align*}
    C^2 \ltwo{\ts_i^\top}^2 \cdot  \err_{1,k+1}(\pug)
        & \le C^2 \frac{d}{k} c_{k+1,\diffp} \frac{k+1}{\diffp} \\
        & = C^2 \frac{d}{\diffp} c_{d,\diffp} \frac{c_{k+1,\diffp}}{c_{d,\diffp}} (1 + 1/k) \\
        & = C^2 \frac{d}{\diffp} c_{d,\diffp} \cdot \left(1 + O\left( \frac{\diffp + \log k}{k} \right) \right) \\
        & = \err_{1,d}(\pug) \cdot \left(1 + O\left( \frac{\diffp + \log k}{k} + \sqrt{\frac{\log^2(k/\delta)}{k}} \right) \right),
    \end{align*}
    where the third step follows from~\Cref{prop:C_eps}. For the second term, we have
    \begin{align*}
    \E[ \ltwo{\sqrt{C} \ts_i^\top u_i[1:k] - v_i}^2]
        & = \E\left[ \ltwo{\sqrt{C} \ts_i^\top u_i[1:k] - \sqrt{C} \ts_i^\top \ts_i v_i + \sqrt{C} \ts_i^\top \ts_i v_i - v_i}^2 \right] \\
        & \le 2C \E\left[ \ltwo{\ts_i^\top u_i[1:k] - \ts_i^\top \ts_i v_i}^2 \right]  + 2 \E\left[ \ltwo{\sqrt{C} \ts_i^\top \ts_i v_i - v_i}^2 \right] \\
        & \le 2C \E\left[ \ltwo{\ts_i^\top u_i[1:k] - \ts_i^\top \ts_i v_i}^2 \right]  + 2 \E\left[ \ltwo{\sqrt{C} \ts_i^\top \ts_i v_i - \ts_i^\top \ts_i v_i}^2\right] \\
        & \quad + 2 \E\left[ \ltwo{\ts_i^\top \ts_i v_i - v_i}^2 \right] \\
        & \le 2C \E\left[ \ltwo{\ts_i^\top u_i[1:k] - \ts_i^\top \ts_i v_i}^2\right] + 2 (\sqrt{C}-1)^2 d/k + 2 (d/k - 1),
    \end{align*}   
    where the second inequality follows since $\E[\ts_i^\top \ts_i]= I$. Now we have
    \begin{align*}
    \E\left[ \ltwo{\ts_i^\top u_i[1:k] - \ts_i^\top \ts_i v_i}^2 \right]
        & = \E\left[ \ltwo{\ts_i^\top \ts_i v_i /\sqrt{C} - \ts_i^\top \ts_i v_i}^2 \mid E \right] \P(E) \\ 
        & \quad + \E\left[ \ltwo{\ts_i^\top \ts_i v_i /\ltwo{\ts_i v_i} - \ts_i^\top \ts_i v_i}^2 \mid E^c \right] \P(E^c) \\
        & \le (1/\sqrt{C}-1)^2 d/k + 2\bar \delta(d/k + (d/k)^2 ) \\
        & \le O\left( \frac{ d \sqrt{\log^2(k/\delta)}}{k^{3/2}}  + \frac{\bar \delta d^2}{k^2} \right).
    \end{align*}   
    Overall, putting these back in Inequality~\eqref{eq:main-err-unb}, we get
    \begin{align*}
    \E \left[\ltwo{\hat \mu - \frac{1}{n} \sum_{i=1}^n v_i}^2 \right]
        & = \frac{1}{n} \err_{1,d}(\pug) \cdot \left(1 + O\left( \frac{\diffp + \log k}{k} + \sqrt{\frac{\log^2(k/\delta)}{k}} \right) \right) \\ 
        & \quad + O \left( \frac{1}{n} \left( \frac{ d \sqrt{\log^2(k/\delta)}}{k^{3/2}}  + \frac{n \delta d^2}{k^2} + \frac{d}{k}  \right)
         +  ( n \delta d/k)^2 \right). 
    \end{align*}
    Noting that $\err_{1,d}(\pug)/n = \err_{n,d}(\pug) = c_{d,\diffp} \cdot \frac{d}{n\diffp}$ for some constant $c_{d,\diffp}$, this implies the theorem given that $\delta = k/n^2 d$.

\end{proof}

%% file: appendix-gaussian.tex
\section{\pjl~using Gaussian transforms}
\label{sec:gauss}
Building on the randomized projection framework of the previous section, in this section we instantiate it with the Gaussian transform. In particular, we sample $\ts \in \R^{k \times d}$ from the Gaussian distribution where $W$ has i.i.d. $\mathcal N(0, 1/k)$ entries. The following theorem states our guarantees for this distribution.
\begin{theorem}
\label{thm:err-gauss}
    Let $k \le d$ and $\ts \in \R^{k \times d}$ be sampled from the Gaussian distribution where $W$ has i.i.d. $\mathcal N(0, 1/k)$ entries.
    Then for all unit vectors $v_1,\ldots,v_n \in \reals^d$, setting $\hat \mu =  \Ajl \left( \Rjl(v_1), \dots, \Rjl(v_n) \right) $, 
    \begin{align*}
    \E\left[\ltwo{ \hat \mu - \frac{1}{n}\sum_{i=1}^n v_i  }^2 \right] 
    \le  \err(\pug_d,n) \left(1 +   O\left( \sqrt{\frac{k}{d}} + \frac{\diffp + \log k}{k} \right) \right)
      + O \left({\frac{1}{k^2}} \right).
    \end{align*}
\end{theorem}

The proof follows directly from~\Cref{thm:err-jl-gen} and the following proposition which proves certain properties of the Guassian transform.
\begin{proposition}
\label{prop:gaussian-prop}
    Consider $W\in\reals^{k\times d}$ with i.i.d. $\mathcal N(0, 1/k)$ entries and a fixed $v\in\reals^d$. Then
    \begin{enumerate}
        \item Bounded operator norm: 
        \begin{equation*}
            \E \|W^\top\|^2 \le \frac{d}{k} \left(1+ O\left(\sqrt{\frac{k}{d}} \right) \right).
        \end{equation*}
        \item Bounded bias: for every unit vector $v \in \reals^d $
        \begin{equation*}
            \left\| \E \frac{\ts^\top \ts v}{\|\ts v\|} - v \right\| = O({1/k}). 
        \end{equation*}
    \end{enumerate}
\end{proposition}

\begin{proof}

For the first item, we rely on standard results in random matrix theory. If we let Z denote the top singular value of $\sqrt{k}W^\top$, then (\citealp{DavidsonS03}, Theorem 2.13) shows that for any $t$, $\Pr(Z > \sqrt{d}+\sqrt{k} + t) < \exp(- t^2 / 2)$. This implies that $median(Z) \leq \sqrt{d}+\sqrt{k} + 2$. Further, by the isoperimetric inequality, $Z$ is concentrated around its median with subGaussian tails, so that the second moment of $Z-median(Z)$ is at most $O(1)$. Thus the second moment of  $Z$ is at most $median(Z)^2 + O(1) \leq (\sqrt{d}+\sqrt{k} + 2)^2 + O(1)$. Scaling this by $k$, we conclude that $\E[\|W^\top\|_{op}^2] \leq \frac{d}{k}(1+ 2\sqrt{\frac{k}{d}} + \frac{O(1)}{k})$.

For the second item, we use a change of variables. Let $\ts' = \ts P^\top$ where $P$ is the rotation matrix such that $P v = e_1$, the first standard basis vector. Recall that the rotation matrix $P$ is orthogonal i.e. $P^\top = P^{-1}$. Due to the rotational symmetry of the normal distribution, $\ts'$ is a random matrix with i.i.d. $\mathcal N(0, 1/k)$ entries. Note that $\ts = \ts' P$.
\begin{align*}
    \E_{\ts}[\ts^\top u] &= \E_{{\ts'}}\left[\frac{1}{\ltwo{\ts' P v}}P^\top{\ts'}^\top \ts' P v\right]\\
        &=P^\top\E\left[\underbrace{\frac{1}{\ltwo{\ts' e_1}} {\ts'}^\top \ts' e_1}_z\right]
\end{align*}
Notice that $z_j = \frac{1}{\ltwo{\ts' e_1}} e_j^\top {\ts'}^\top \ts' e_1 = \langle \ts' e_j, \frac{1}{\ltwo{\ts' e_1}} \ts' e_1\rangle$. Because $\ts'$ has i.i.d. $\mathcal N(0, 1/k)$ entries, $z_1 = \ltwo{\ts' e_1}$ is $1/\sqrt{k}$ times a $\chi$ random variable with $k$ degrees of freedom. We let $\frac{1}{c} = \E[z_1] = \frac{1}{\sqrt{k}}\cdot \frac{\sqrt{2}\Gamma((k+1)/2)}{\Gamma(k/2)} = 1- O(1/k)$ and $\E[z_j]=0~\forall j>1$. Thus, $\E[z] = \frac{1}{c}e_1$ and $\E[\ts_i^\top u_i] = \frac{1}{c}P_i^\top e_1 = \frac{1}{c}P^\top P v = \frac{1}{c}v$. Therefore, $\ltwo{\E[\frac{\ts^\top \ts v}{\ltwo{\ts v}} - v} = \ltwo{|1/c-1| v} = O(1/k)$.
\end{proof}

%% file: appendix-rotation.tex
\subsection{Helper lemmas for random rotations}

\begin{lemma}
\label{lemma:unit-proj}
Let $x$ be a random unit vector on the unit ball of $\reals^d$ and $z$ be the projection of $x$ on to the last $k$ coordinates. We have
\[
    \left|\E[\|z\|]-\sqrt{k/d}\right| = O\left(\frac{1}{\sqrt{kd}}\right)
\]
\end{lemma}

\begin{proof}
We represent $d$ dimensional vector $x$ using spherical coordinates as follows.

\begin{align*}
x_{1} & =\cos\left(\phi_{1}\right)\\
x_{2} & =\sin\left(\phi_{1}\right)\cos\left(\phi_{2}\right)\\
 & \ldots\\
x_{d-1} & =\sin\left(\phi_{1}\right)\cdots\sin\left(\phi_{d-2}\right)\cos\left(\phi_{d-1}\right)\\
x_{d} & =\sin\left(\phi_{1}\right)\cdots\sin\left(\phi_{d-2}\right)\sin\left(\phi_{d-1}\right)
\end{align*}

The squared length of the projection is

\[
\sum_{i=d-k+1}^{d}x_{i}^{2}=\sin^{2}\left(\phi_{1}\right)\cdots\sin^{2}\left(\phi_{d-k}\right)
\]

Recall the surface area element is $\sin^{d-2}\left(\phi_{1}\right)\sin^{d-3}\left(\phi_{2}\right)\cdots\sin\left(\phi_{d-2}\right)d\phi_{1}\cdots d\phi_{d-1}$.

For $k\ge2$, the expected length is

\[
\frac{\int_{0}^{\pi}\cdots\int_{0}^{\pi}\int_{0}^{2\pi}\left(\sin\left(\phi_{1}\right)\cdots\sin\left(\phi_{d-k}\right)\right)\sin^{d-2}\left(\phi_{1}\right)\sin^{d-3}\left(\phi_{2}\right)\cdots\sin\left(\phi_{d-2}\right)d\phi_{1}\cdots d\phi_{d-1}}{S_{d-1}}
\]

where $S_{d-1}$ is the surface area of the unit ball, which is $S_{d-1}=\frac{2\pi^{d/2}}{\Gamma\left(d/2\right)}$.

We first evaluate the integral for each sine power.
\begin{claim}
For integer $n \ge 1$ we have 
\begin{align*}
\int_{0}^{\pi}\sin^{n}xdx & =\frac{\Gamma((n+1)/2)}{\Gamma\left(1+n/2\right)}\sqrt{\pi}
\end{align*}
\end{claim}
\begin{proof}
For $n\ge 2$, we have
\begin{align*}
    \int\sin^{n}xdx &= -\int\sin^{n-1}x d(\cos x)\\
    &= -\sin^{n-1}x\cos x + (n-1)\int\sin^{n-2}x \cos^2 x dx\\
    &= -\sin^{n-1}x\cos x + (n-1)\int\sin^{n-2}x (1-\sin^2 x) dx\\
\end{align*}
Thus,
\begin{align*}
    \int_0^{\pi}\sin^{n}xdx &= \frac{n-1}{n}\int_0^{\pi}\sin^{n-2}xdx-\frac{\sin^{n-1}x\cos x}{n}\Big|_0^{\pi}\\
    & =\frac{(n-1)/2}{n/2}\int_{0}^{\pi}\sin^{n-2}xdx
\end{align*}
The claim then follows using induction with base cases $\int_{0}^{\pi}\sin xdx=2$ and $\int_{0}^{\pi}dx=\pi$.
\end{proof}

\begin{align*}
 & \int_{0}^{\pi}\cdots\int_{0}^{\pi}\int_{0}^{2\pi}\left(\sin\left(\phi_{1}\right)\cdots\sin\left(\phi_{d-k}\right)\right)\sin^{d-2}\left(\phi_{1}\right)\sin^{d-3}\left(\phi_{2}\right)\cdots\sin\left(\phi_{d-2}\right)d\phi_{1}\cdots d\phi_{d-1}\\
= & 2\pi\int_{0}^{\pi}\cdots\int_{0}^{\pi}\sin^{d-1}\left(\phi_{1}\right)\cdots\sin^{k}\left(\phi_{d-k}\right)\sin^{k-2}\left(\phi_{d-k+1}\right)\cdots\sin\left(\phi_{d-2}\right)d\phi_{1}\cdots d\phi_{d-2}\\
= & 2\pi\frac{\Gamma(d/2)}{\Gamma\left((d+1)/2\right)}\sqrt{\pi}\cdots\frac{\Gamma((k+1)/2)}{\Gamma\left(\left(k+2\right)/2\right)}\sqrt{\pi}\cdot\frac{\Gamma((k-1)/2)}{\Gamma\left(k/2\right)}\sqrt{\pi}\cdots\frac{\Gamma\left(1\right)}{\Gamma\left(3/2\right)}\sqrt{\pi}\\
= & 2\pi^{d/2}\frac{\Gamma\left(\left(k+1\right)/2\right)}{\Gamma\left(\left(d+1\right)/2\right)\Gamma\left(k/2\right)}
\end{align*}

The expected length is 
\begin{align*}
    \E[\|z\|] &=\frac{\Gamma\left(\left(k+1\right)/2\right)\Gamma\left(d/2\right)}{\Gamma\left(\left(d+1\right)/2\right)\Gamma\left(k/2\right)}\\
    &=\frac{\Gamma(k)2^d \Gamma(d/2)^2}{2^k \Gamma(k/2)^2 \Gamma(d)}\\
    &=\frac{\sqrt{k}\left(1-\frac{1}{4k}+O\left(1/k^{2}\right)\right)}{\sqrt{d}\left(1-\frac{1}{4d}+O\left(1/d^{2}\right)\right)}
\end{align*}
In the second line, we use the Legendre duplication formula $\Gamma(k/2)\Gamma((k+1)/2)=2^{1-k}\sqrt{\pi}\Gamma(k)$. In the third line, we use the Stirling's approximation $\Gamma(z) = \sqrt{2\pi/z} (z/e)^z (1+1/(12z) + O(1/z^2))$.
\end{proof}

%% file: appendix-srht.tex
\subsection{SRHT Analysis}
The Subsampled Randomized Hadamard Transform (SRHT) is the random matrix ensemble defined as $W = \sqrt{\frac dk} SHD$. Here $D\in\reals^{d\times d}$ is a diagonal matrix with independent uniform $\pm 1$ values on its diagonal, and $H\in\reals^{d\times d}$ is the normalized Hadamard transform ($H_{i,j} = (-1)^{\langle v(i),v(j)\rangle}/\sqrt d$, where $v(i)$ is the $(\log_2 d)$-dimensional vector obtained by writing $i$ in binary). The matrix $S\in\reals^{k\times d}$ is a sampling matrix. The fact that the SRHT preserves the Euclidean norm of any fixed vector with large probability has been known for some time \cite{AilonCh09,CohenNW16,Sarlos06}, though different works have analyzed slightly different variants of the SRHT, all having to do with how $S$ is defined.

In this work, we make use of the SRHT in which $S$ samples without replacement: that is, each row of $S$ has a $1$ in a uniformly random entry and zeroes elsewhere, and no two rows of $S$ are equal. The tightest known analysis of the SRHT \cite{CohenNW16} analyzes the SRHT with a different sampling matrix: $S_\eta = \mathop{diag}(\eta)$, where $\eta_1,\ldots,\eta_d$ are independent Bernoulli random variables each with expectation $k/d$ (so that we sample a {\it random} number of rows from $HD$, which is equal to $k$ only in expectation).

The following is a special case of Theorem 9 in the full version of \cite{CohenNW16}

\begin{theorem}[{\cite{CohenNW16}}]\label{thm:cnw}
    Suppose $W = \sqrt{\frac dk}S_\eta HD$ for $S=\mathop{diag}(\eta)$, where $\eta_1,\ldots,\eta_d$ is a sequence of independent, uniform Bernoulli random variables each with expectation $k/d$. Then for some constant $C>0$, for any fixed $u\in\reals^d$ of unit Euclidean norm and $\delta\in(0,1)$,
    $$
    \Pr_{\eta,D}(| \|W u\|_2^2 - 1 | > C\sqrt{\log(1/\delta)\log(k/\delta)/k}) < \delta
    $$
\end{theorem}

An analysis of the SRHT using sampling without replacement then follows as a corollary.

\begin{corollary}\label{cor:cnw}
Suppose $W = \sqrt{\frac dk} SHD$ is obtained with $S$ being a $k\times d$ sampling matrix without replacement. Then for some constant $C>0$, for any fixed $u\in\reals^d$ of unit Euclidean norm and $\delta\in(0,1)$,
 $$
    \Pr_{\eta,D}(| \|W u\|_2^2 - 1 | > C\sqrt{\log^2(k/\delta)/k}) < \delta
    $$
\end{corollary}
\begin{proof}
Consider $W' = \sqrt{\frac dk}S_\eta HD$ with Bernoulli parameter $k/d$, as in \cref{thm:cnw}. Then for any $\delta' \in(0,1)$ and fixed unit vector $u\in\reals^d$, $\Pr_{\eta,D}(E) < \delta'$, where $E$ is the event that $| \|W'u\|_2^2 - 1| > C\sqrt{\log(1/\delta')\log(k/\delta')/k}$. But we also have 
\begin{align*}
\Pr(E) &\ge \Pr(E\cap (\|\eta\|_1 = k))\\
{}&= \Pr(E \mid \|\eta\|_1 = k)\cdot \Pr(\|\eta\|_1 = k) \\
{}&= \Pr(E \mid \|\eta\|_1 = k) \cdot \Theta(1/\sqrt k) .
\end{align*}

Note $\Pr(E \mid \|\eta\|_1 = k)$ is exactly $\Pr(| \|Wu\|_2^2 - 1| > C\sqrt{\log(1/\delta')\log(k/\delta')/k})$, where $W$ is defined by sampling without replacement. Thus we have
$$
\Pr(| \|Wu\|_2^2 - 1| > C\sqrt{\log(1/\delta')\log(k/\delta')/k}) < C \delta' \sqrt{k} .
$$
The claim then follows by applying the above with $\delta' = \delta/(C\sqrt k)$.
\end{proof}

%% file: appendix-compressed-privunit.tex
\section{Compressed \pug}

Compressing the \pu\ (resp. \pug) algorithm, using the technique of~\citet{FeldmanTa21}, requires a pseudorandom generator that generates samples from a unit ball (resp. Gaussian) and fools spherical caps. As observed in~\citep{FeldmanTa21}, such PRGs with small seed length are known~\citep{KothariM15,GopalanKM15}. However, the constructions in those works are optimized for seed length, and the computational cost of expanding a seed to a vector is a large polynomial. In this section, we argue that for inputs $x$ having $b$ bits of precision, we can compress \pu/\pug\ to small seed length with a relatively efficient algorithm for seed expansion. 

We will rely on Nisan's generator~\cite{Nisan92} which says that any space $S$ computation that consumes $N$ bits  of randomness can be $\delta$-fooled using a random seed of length $O(\log N (S+\log N/\delta))$. Moreover, the computational cost of generating a pseudorandom string from a random seed is $O(N\log N)$. In our set up, the test that privacy of \pu/\pug\ depends on is the $[g\cdot x \geq \gamma]$, when $g$ is chosen from the Gaussian distribution. This probability that this test passes for the Gaussian distribution is $e^{-c\varepsilon}$ for some constant $c$, and thus it suffices to set $\delta$ to be $e^{-c\varepsilon}\beta$ to ensure that mechanism satisfies $(\varepsilon+2\beta)$-DP. For the rest of this discussion, we will set $\beta = \varepsilon\tau/2$ which leads to $\varepsilon' < \varepsilon(1+\tau)$. We can set $\tau$ to be inverse polynomial as the dependence of the parameters on $\tau$ will be logarithmic.

The test of interest for us can be implemented in $S=O(\log d + b)$ space, and requires $N=db$ bits of randomness. Plugging in these values and $\delta = \eps \tau e^{-c\eps}/2$, we get seed length $O(\log db (b+c\eps+\log db/\eps\tau)$ and each expansion from seed to value requires run time $O(db \log db)$. 
Each run of \pug\ requires $O(e^{c\varepsilon})$ expected random strings, leading to a run time of $O(e^{c\epsilon}bd\log db)$. 

For our algorithm, we run this on a $k$-dimensional vector instead of a $d$-dimensional one, with $b = \log d$. This gives us seed length $O(\log (k\log d) \cdot (\log d + \eps+\log (k\log d/\eps)))$. Given the projected vector, the  run time is $O(e^{c\eps}k\log^2 d)$.